\documentclass{article}

\PassOptionsToPackage{numbers, compress}{natbib}

\usepackage[main, final]{neurips_2026}

\usepackage{amsmath,amsfonts,bm}

\def\eqref#1{equation~\ref{#1}}

\def\1{\bm{1}}

\DeclareMathAlphabet{\mathsfit}{\encodingdefault}{\sfdefault}{m}{sl}
\SetMathAlphabet{\mathsfit}{bold}{\encodingdefault}{\sfdefault}{bx}{n}

\newcommand{\R}{\mathbb{R}}

\usepackage[T1]{fontenc}    %
\usepackage{textcomp}       %
\usepackage{hyperref}       %
\usepackage{url}            %
\usepackage{booktabs}       %
\usepackage{amsmath}        %
\usepackage{amsthm}         %
\usepackage{amsfonts}       %
\usepackage{nicefrac}       %
\usepackage[nopatch=footnote]{microtype} %
\usepackage[table]{xcolor}  %
\usepackage{subcaption}     %
\usepackage{multirow}       %

\usepackage{xspace}
\usepackage{dsfont}
\usepackage{afterpage}

\usepackage{enumitem}
\usepackage{caption} %
\usepackage{colortbl} %
\usepackage{bbm}
\usepackage{algorithm}
\usepackage{algpseudocode}                  %
\usepackage{algxpar}                       %
\floatstyle{ruled}
\restylefloat{algorithm}
\usepackage{setspace}                       %
\usepackage{makecell}

\usepackage{wrapfig}
\usepackage{float}
\usepackage[section]{placeins}
\usepackage{needspace}

\usepackage{mathtools} %
\usepackage{soul} %
\usepackage{csquotes}

\usepackage{tikz}
\usepackage{bm}
\usetikzlibrary{shapes.geometric}

\definecolor{LightCyan}{rgb}{0.88,1,1}
\definecolor{mygreen}{RGB}{0, 150, 100}
\newcommand{\res}[2]{#1\color{gray}{\scriptsize{$\pm$#2}}}
\newcommand{\best}[3]{\textbf{#1}\color{gray}{\scriptsize{$\pm$#2}}\color{mygreen}{$^{\textbf{$\Delta$\scriptsize{#3}}}$}}

\definecolor{customred}{HTML}{d62728}
\definecolor{customblue}{HTML}{1f77b4}
\newcommand{\graydot}{\tikz[baseline=-0.6ex]\draw[gray, fill=gray, thick] (0,0) circle (0.35ex);}
\newcommand{\bluetriangle}{\tikz[baseline=-0.5ex]\fill[customblue] (0,1.0ex) -- (-0.866ex,-0.5ex) -- (0.866ex,-0.5ex) -- cycle;}
\newcommand{\redstar}{\tikz[baseline=-0.7ex]\node[star, star points=5, star point ratio=2, fill=customred, inner sep=1.3pt, anchor=center] {};}

\usepackage{pifont}%

\sethlcolor{aliceblue}

\definecolor{citecolor}{HTML}{2980b9}
\definecolor{linkcolor}{HTML}{c0392b}
\definecolor{urlcolor}{RGB}{157,49, 251}

\hypersetup{
    breaklinks=true,
    colorlinks=true,
    citecolor=citecolor,
    linkcolor=linkcolor,
    urlcolor=urlcolor
}

\newtheorem{proposition}{Proposition}
\newcommand{\motivationdivider}{\hspace{0.012\linewidth}{\color{gray!70}\rule{0.6pt}{0.12\textheight}}\hspace{0.012\linewidth}}

\newcommand{\LockedMotivationFigure}{%
\begin{figure}[htbp]
\centering
\captionsetup{skip=3pt}
\captionsetup{font=small}
\begin{subfigure}[t]{0.31\linewidth}
\centering
\includegraphics[width=\linewidth]{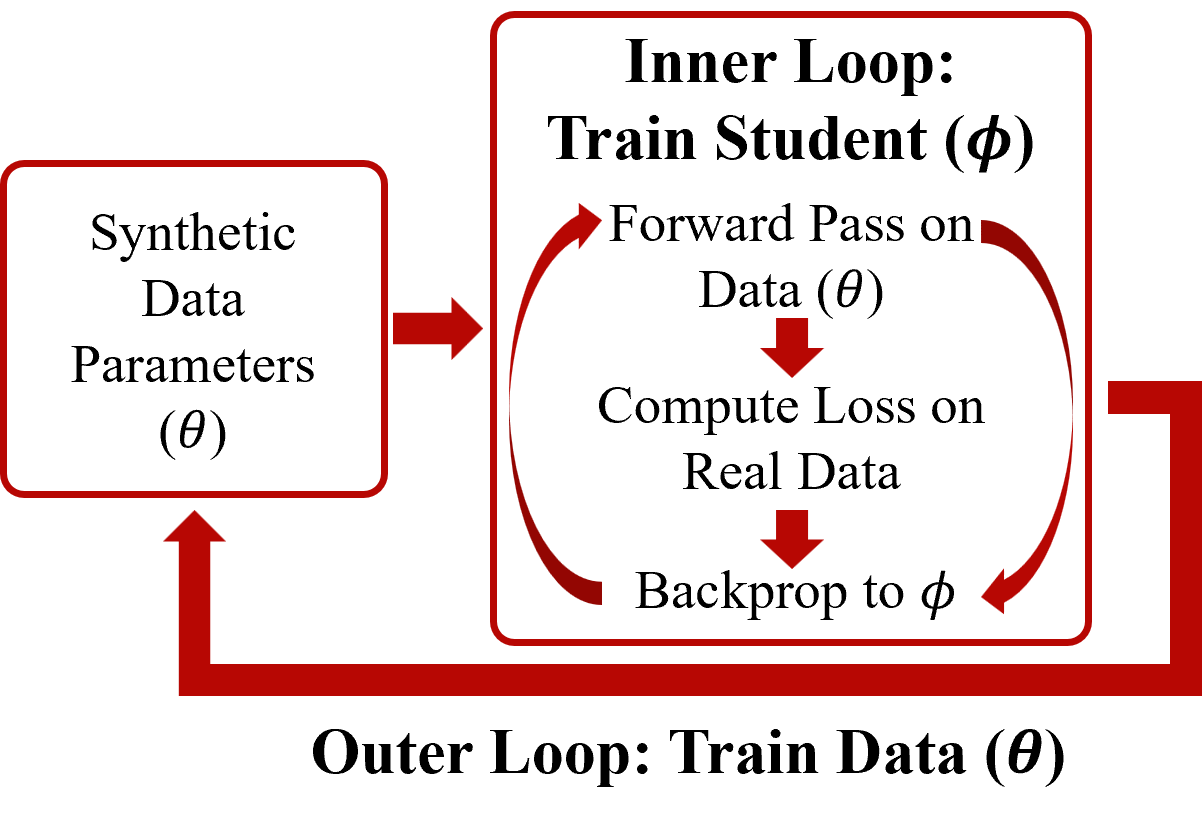}
\caption{Optimization-based methods}
\label{fig:motivation_optimization}
\end{subfigure}%
\motivationdivider
\begin{subfigure}[t]{0.31\linewidth}
\centering
\includegraphics[width=\linewidth]{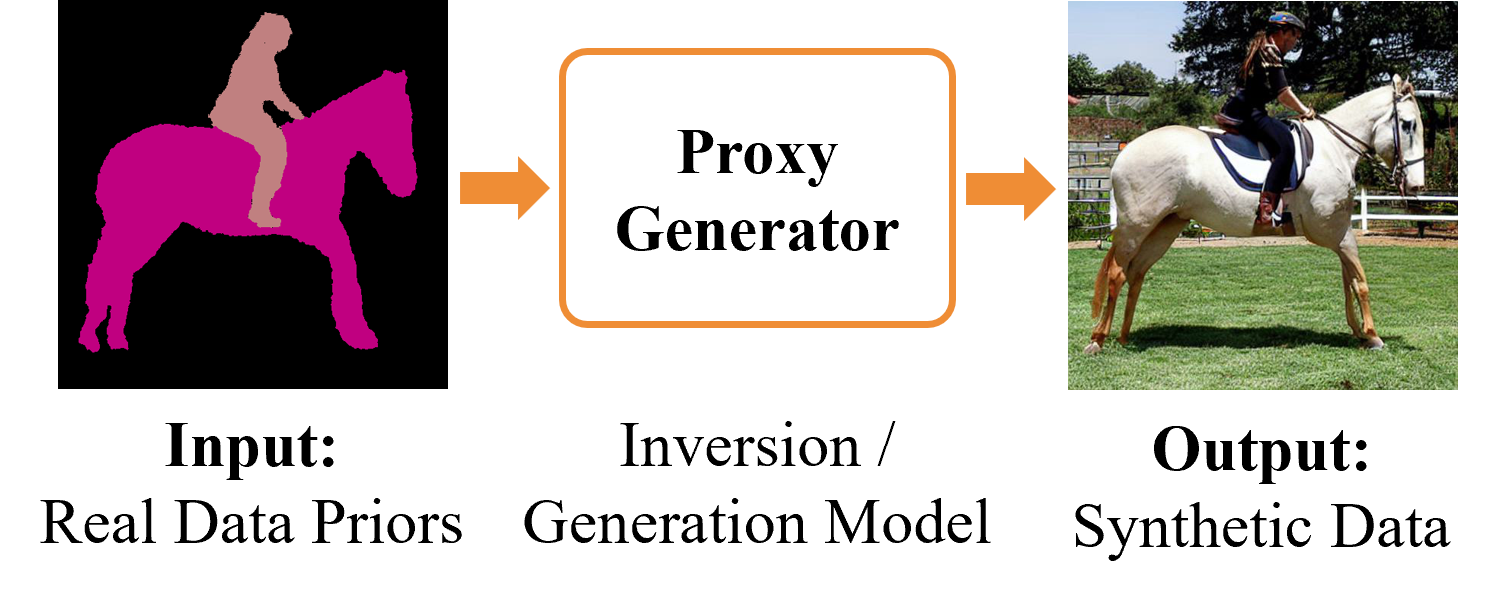}
\caption{Proxy-based methods}
\label{fig:motivation_proxy}
\end{subfigure}%
\motivationdivider
\begin{subfigure}[t]{0.31\linewidth}
\centering
\includegraphics[width=\linewidth]{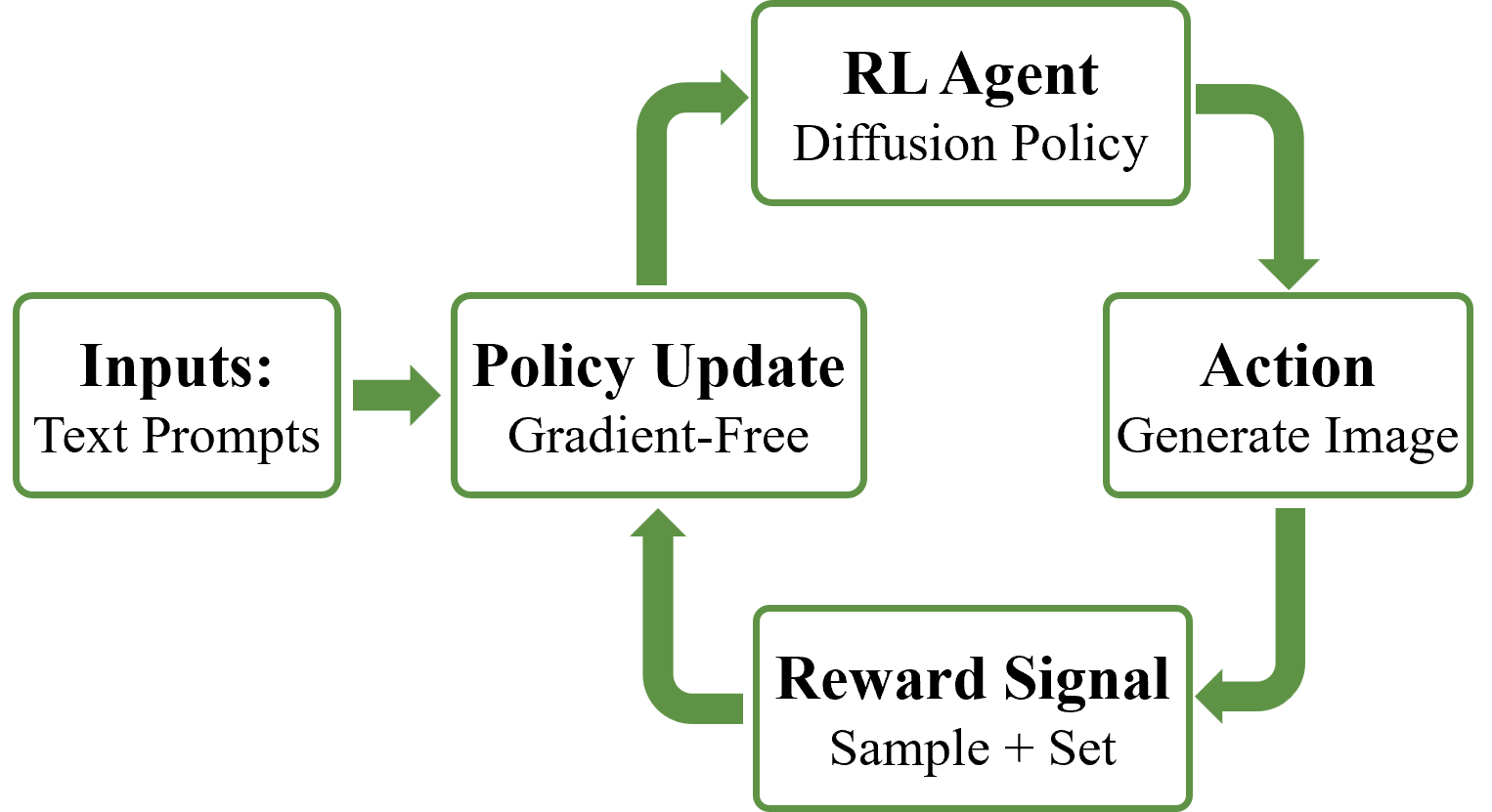}
\caption{S$^2$-RL (Ours)}
\label{fig:motivation_ours}
\end{subfigure}%
\caption{Comparison of paradigms. (Left) Optimization-based methods incur prohibitive memory costs from unrolling pixel-wise gradients. (Middle) Proxy-based generative methods require masks from real data, restricting synthesized datasets to existing layouts. (Right) S$^2$-RL treats distillation as active scene construction, using RL to generate diverse spatial topologies and valid scene layouts from text alone.}
\label{fig:motivation}
\end{figure}
}

\newcommand{\LockedMethodFigure}{%
\captionsetup{font=small}
\begin{figure}[t]
\centering
\includegraphics[width=0.98\linewidth]{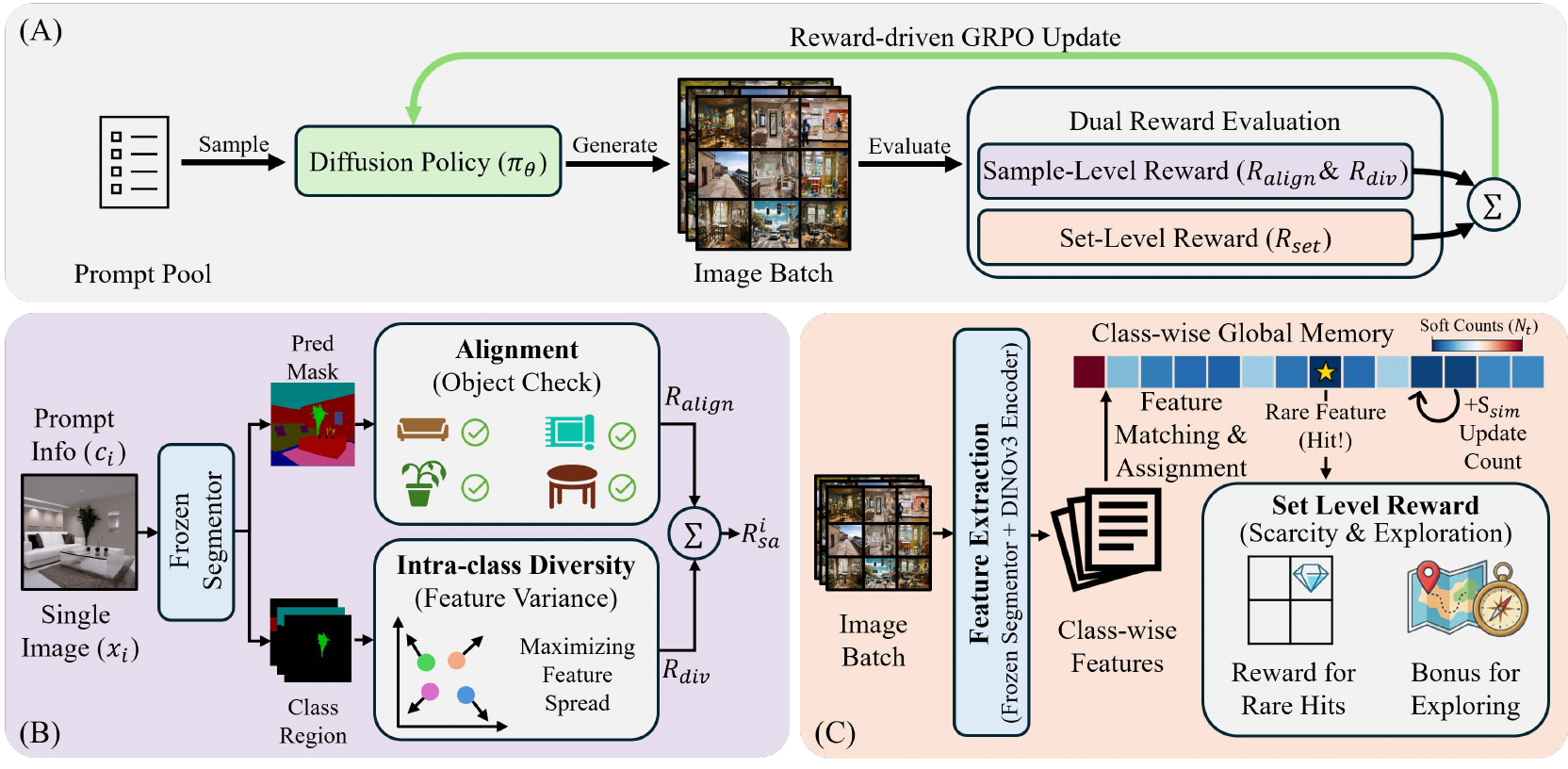}
\caption{The architecture of our proposed method. (A) The framework utilizes GRPO to finetune the diffusion policy with a dual reward evaluation system. (B) Sample-level reward ensures semantic alignment and local variance by evaluating individual images against prompt constraints. (C) Set-level reward encourages global exploration. By mapping batch features to a global prototype memory, the system identifies and rewards low-frequency prototypes, assigning high bonuses for scarcity and exploration to mitigate mode collapse.}
\label{fig:method}
\end{figure}
}

\newcommand{\LockedAdeTable}{%
\captionsetup{font=small}
\begin{table}[H]
\centering
\setlength{\tabcolsep}{4pt}
\caption{Quantitative comparison on ADE20K \cite{ade20k} using SegFormer-B5 \cite{segformer} as the downstream student. We report mAcc and mIoU (\%) at various distillation ratios. $^\ast$ denotes our reproduced implementation. The \colorbox{LightCyan}{highlighted row} denotes our method, and \textcolor{mygreen}{green} indicates the improvement over the strongest baseline.}
\label{tab:ade20k}
\resizebox{\linewidth}{!}{%
\begin{tabular}{lcccccc}
\toprule
\multirow{2.5}{*}{\textbf{Method}} & \multicolumn{2}{c}{\textbf{0.25\%}} & \multicolumn{2}{c}{\textbf{0.5\%}} & \multicolumn{2}{c}{\textbf{1.0\%}} \\
\cmidrule(lr){2-3} \cmidrule(lr){4-5} \cmidrule(lr){6-7}
& \textbf{mAcc} & \textbf{mIoU} & \textbf{mAcc} & \textbf{mIoU} & \textbf{mAcc} & \textbf{mIoU} \\
\midrule
Random           & \res{14.48}{0.46} & \res{9.59}{0.36}  & \res{21.11}{0.69} & \res{14.92}{0.45} & \res{27.08}{1.11} & \res{19.95}{0.67} \\
Uniform          & \res{18.48}{1.27} & \res{12.05}{0.68} & \res{27.48}{0.93} & \res{18.59}{0.47} & \res{34.95}{1.17} & \res{24.10}{0.55} \\
K-Center         & \res{18.02}{0.55} & \res{12.56}{0.55} & \res{24.74}{1.22} & \res{17.92}{0.93} & \res{31.69}{1.21} & \res{23.90}{0.96} \\
Herding          & \res{15.51}{0.23} & \res{10.54}{0.17} & \res{21.29}{0.27} & \res{15.44}{0.21} & \res{27.68}{0.25} & \res{21.00}{0.21} \\
\midrule
SD3.5 Baseline   & \res{20.34}{0.91} & \res{13.54}{0.56} & \res{25.92}{1.06} & \res{18.09}{0.90} & \res{30.05}{0.70} & \res{21.66}{0.40} \\
UniDD$^\ast$~\cite{ff2} & \res{17.83}{1.12} & \res{12.05}{0.64}  & \res{25.90}{1.07} & \res{17.81}{0.90} & \res{31.11}{0.83} & \res{22.70}{0.56} \\
\rowcolor{LightCyan}\textbf{S$^2$-RL (Ours)} & \best{23.07}{0.92}{2.73} & \best{15.60}{0.59}{2.06} & \best{30.74}{0.73}{3.26} & \best{21.52}{0.56}{2.93} & \best{35.58}{0.57}{0.63} & \best{25.47}{0.24}{1.37} \\
\midrule
\textit{\textcolor{gray}{Whole Dataset}} & \multicolumn{6}{c}{\textit{\textcolor{gray}{mAcc: 60.69 \hspace{10em} mIoU: 49.13}}} \\
\bottomrule
\end{tabular}%
}
\end{table}
}

\newcommand{\LockedTsneFigure}{%
\captionsetup{font=small}
\begin{figure}[H]
\centering
\begin{subfigure}[t]{0.24\linewidth}
\centering
\includegraphics[width=\linewidth]{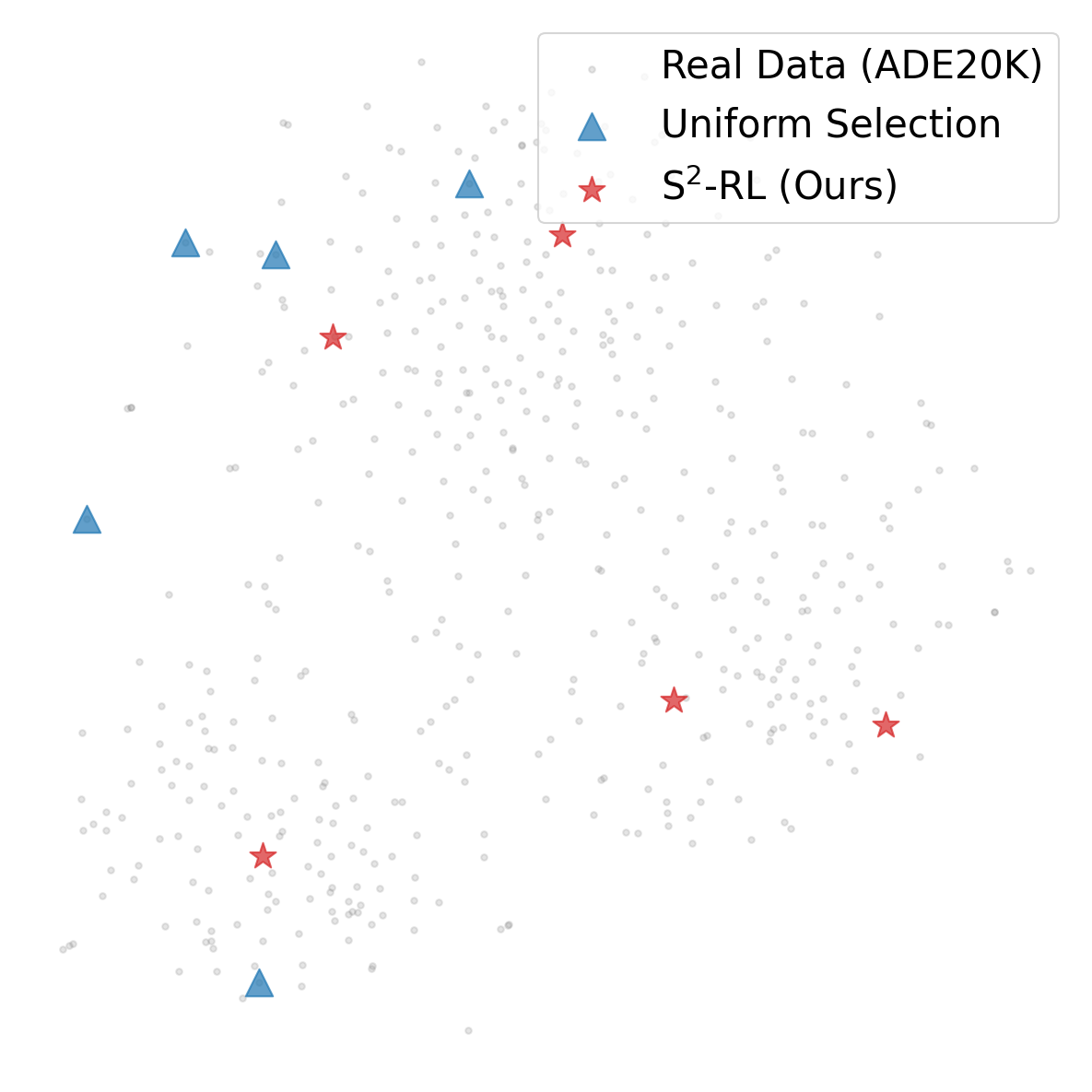}
\caption{seat}
\label{fig:tsne_seat}
\end{subfigure}%
\hfill
\begin{subfigure}[t]{0.24\linewidth}
\centering
\includegraphics[width=\linewidth]{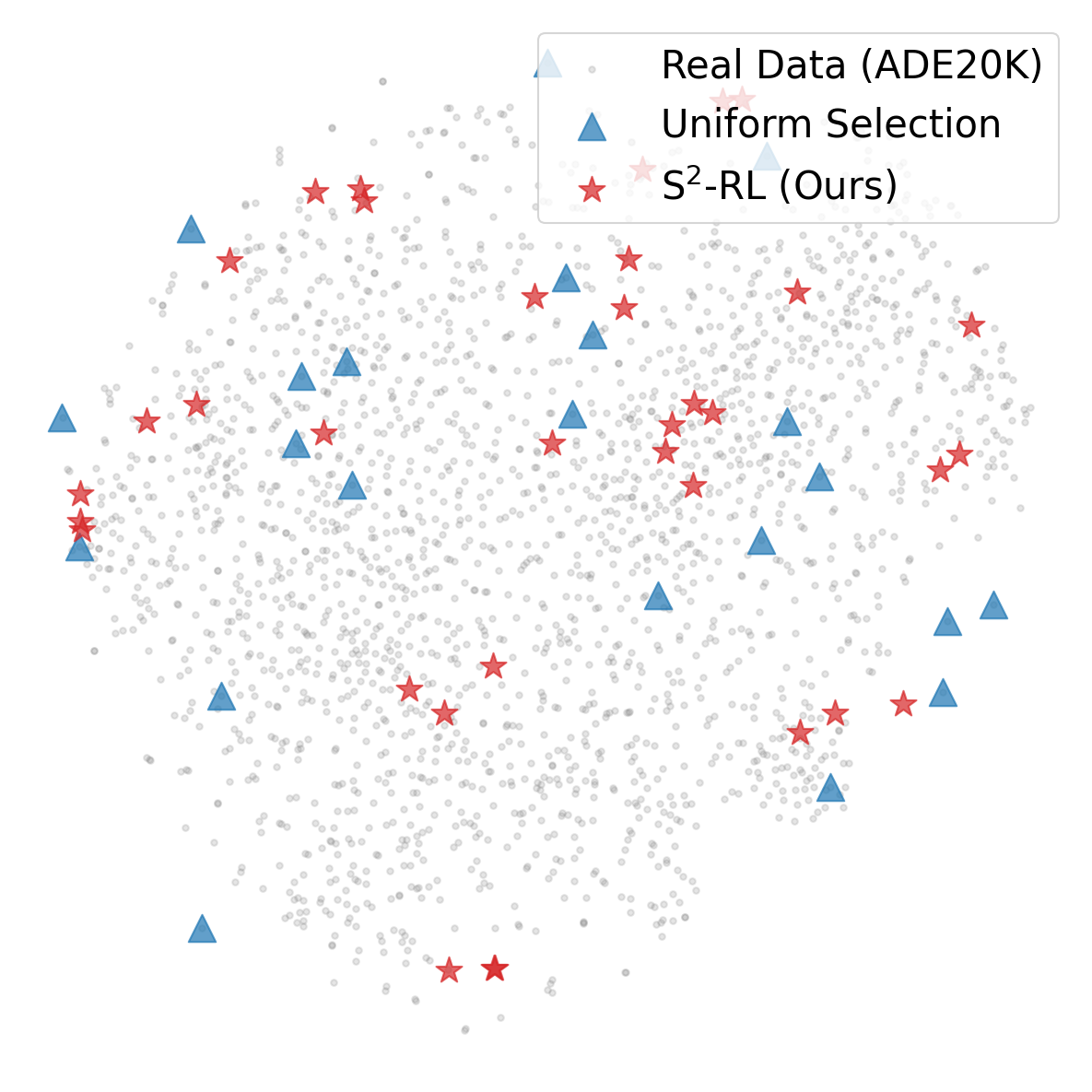}
\caption{grass}
\label{fig:tsne_grass}
\end{subfigure}%
\hfill
\begin{subfigure}[t]{0.24\linewidth}
\centering
\includegraphics[width=\linewidth]{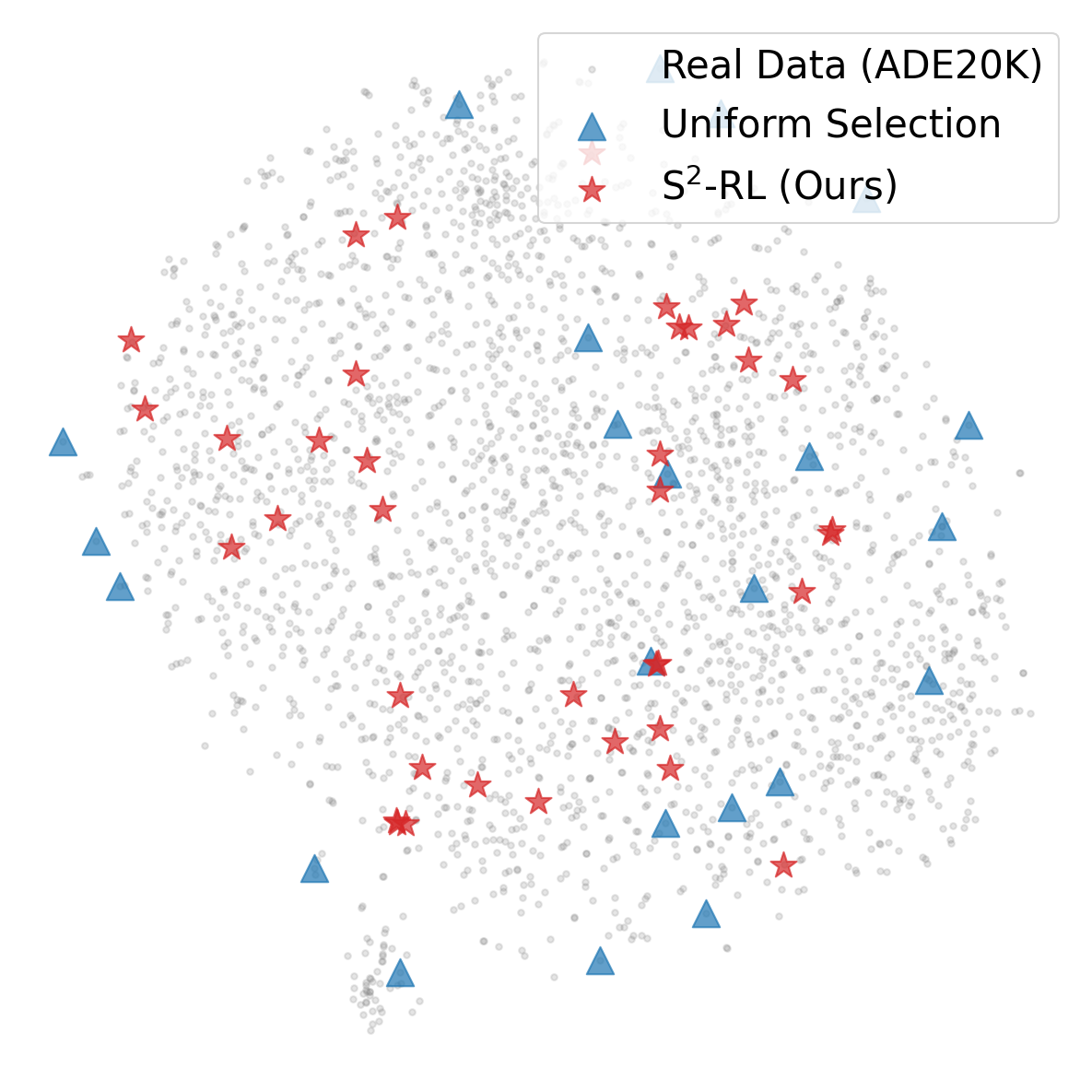}
\caption{car}
\label{fig:tsne_car}
\end{subfigure}%
\hfill
\begin{subfigure}[t]{0.24\linewidth}
\centering
\includegraphics[width=\linewidth]{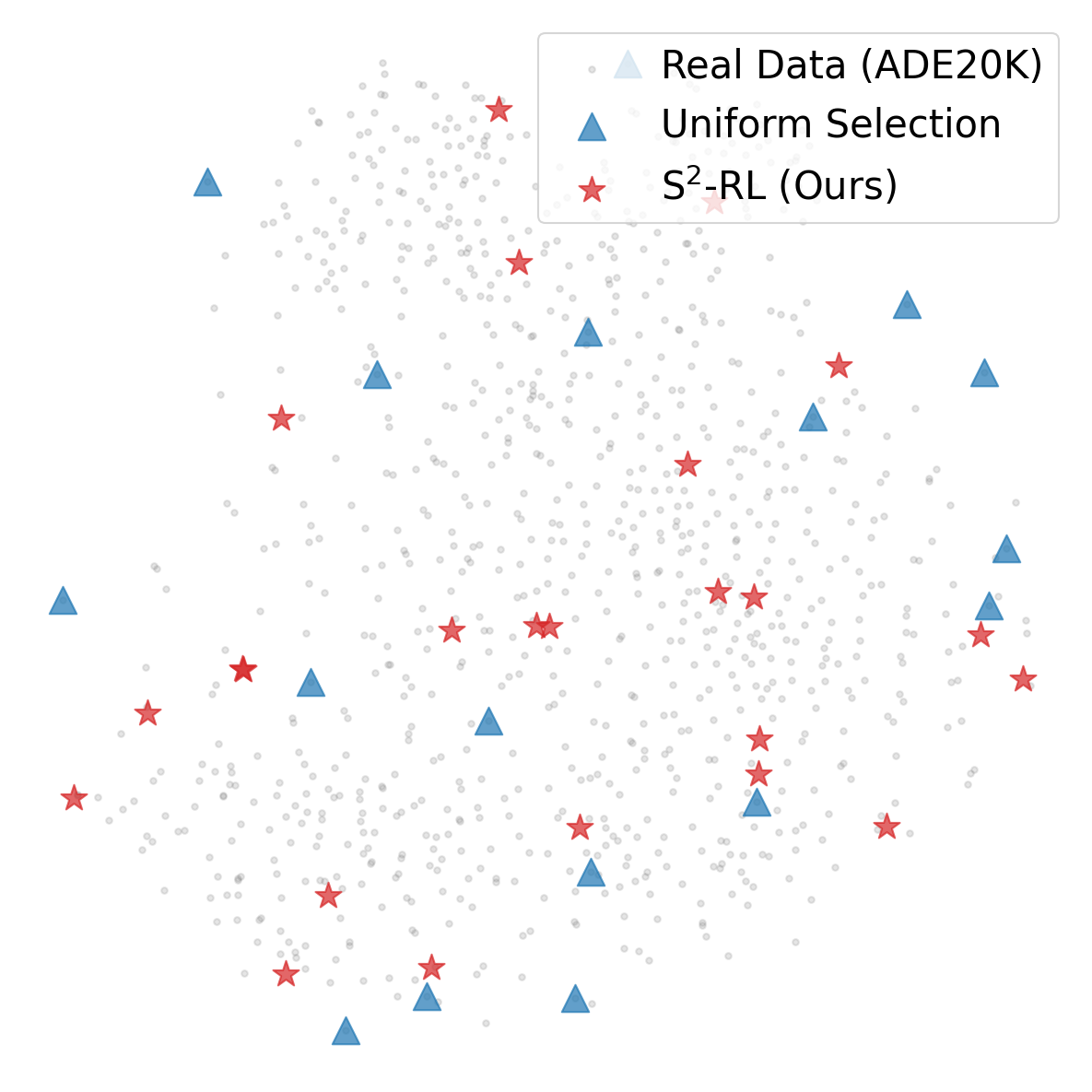}
\caption{sink}
\label{fig:tsne_sink}
\end{subfigure}
\caption{\textbf{Class-wise t-SNE embeddings under the 1\% data ratio.} Compared with Uniform Selection (\protect\bluetriangle), S$^2$-RL (\protect\redstar) spans broader regions of the ADE20K feature manifold (\protect\graydot) across all four categories.}
\label{fig:tsne}
\end{figure}
}

\newcommand{\LockedCocoTable}{%
\captionsetup{font=small}
\begin{table}[t]
\centering
\setlength{\tabcolsep}{4pt}
\caption{Quantitative comparison on COCO \cite{coco} using SegFormer-B5 \cite{segformer} as the downstream student. We report mAcc and mIoU (\%) at various distillation ratios. $^\ast$ denotes our reproduced implementation. The \colorbox{LightCyan}{highlighted row} denotes our method, and \textcolor{mygreen}{green} indicates the improvement over the strongest baseline.}
\label{tab:coco}
\resizebox{\linewidth}{!}{%
\begin{tabular}{lcccccc}
\toprule
\multirow{2.5}{*}{\textbf{Method}} & \multicolumn{2}{c}{\textbf{0.625\textperthousand}} & \multicolumn{2}{c}{\textbf{1.25\textperthousand}} & \multicolumn{2}{c}{\textbf{2.5\textperthousand}} \\
\cmidrule(lr){2-3} \cmidrule(lr){4-5} \cmidrule(lr){6-7}
& \textbf{mAcc} & \textbf{mIoU} & \textbf{mAcc} & \textbf{mIoU} & \textbf{mAcc} & \textbf{mIoU} \\
\midrule
Random           & \res{28.82}{1.32} & \res{19.71}{0.67} & \res{38.00}{0.40} & \res{27.75}{0.46} & \res{44.89}{0.40} & \res{34.31}{0.37} \\
Uniform          & \res{33.45}{0.83} & \res{23.37}{0.97} & \res{43.37}{0.97} & \res{32.18}{0.75} & \res{50.07}{0.43} & \res{37.39}{0.46} \\
K-Center         & \res{34.77}{0.67} & \res{24.99}{0.57} & \res{43.35}{0.72} & \res{32.30}{0.46} & \res{48.92}{0.47} & \res{37.19}{0.38} \\
Herding          & \res{32.07}{0.35} & \res{23.04}{0.24} & \res{40.20}{0.25} & \res{30.49}{0.22} & \res{47.56}{0.23} & \res{37.10}{0.20} \\
\midrule
SD3.5 Baseline   & \res{38.18}{0.73} & \res{25.98}{0.67} & \res{46.65}{0.48} & \res{32.88}{0.87} & \res{51.81}{0.35} & \res{37.74}{0.34} \\
UniDD$^\ast$~\cite{ff2} & \res{31.45}{0.73} & \res{21.42}{0.72} & \res{41.03}{0.97} & \res{30.02}{0.74} & \res{46.35}{0.97} & \res{34.80}{0.88} \\
\rowcolor{LightCyan}\textbf{S$^2$-RL (Ours)} & \best{41.13}{0.67}{2.95} & \best{28.46}{0.48}{2.48} & \best{48.48}{0.87}{1.83} & \best{35.03}{0.68}{2.15} & \best{53.03}{0.54}{1.22} & \best{38.83}{0.59}{1.09} \\
\midrule
\textit{\textcolor{gray}{Whole Dataset}} & \multicolumn{6}{c}{\textit{\textcolor{gray}{mAcc: 73.39 \hspace{10em} mIoU: 61.04}}} \\
\bottomrule
\end{tabular}%
}
\end{table}
}

\newcommand{\LockedArchTable}{%
\captionsetup{font=small}
\begin{table}[t]
\centering
\caption{Comprehensive robustness analysis on ADE20K \cite{ade20k}. We report mIoU (\%) across varying model capacities, backbone architectures, and semantic segmentation mechanisms at three distillation ratios.}
\label{tab:ablation-arch}
\resizebox{\linewidth}{!}{%
\begin{tabular}{lccccccc}
\toprule
\multirow{2.5}{*}{\textbf{Model Configuration}} & \multicolumn{2}{c}{\textbf{Ratio = 0.25\%}} & \multicolumn{2}{c}{\textbf{Ratio = 0.5\%}} & \multicolumn{2}{c}{\textbf{Ratio = 1.0\%}} & \multirow{2.5}{*}{\textit{\textbf{\textcolor{gray}{Full}}}} \\
\cmidrule(lr){2-3} \cmidrule(lr){4-5} \cmidrule(lr){6-7}
& Uniform & \textbf{S$^2$-RL (Ours)} & Uniform & \textbf{S$^2$-RL (Ours)} & Uniform & \textbf{S$^2$-RL (Ours)} & \\
\midrule
\multicolumn{8}{l}{\cellcolor{gray!10}\textit{\textbf{(a) Impact of Model Capacity}}} \\
ResNet-18 + UPerNet  & \res{5.32}{0.20} & \best{5.90}{0.29}{0.58} & \res{8.33}{0.42} & \best{9.12}{0.36}{0.79} & \res{12.10}{0.61} & \best{12.45}{0.25}{0.35} & \textit{\textcolor{gray}{35.58}} \\
ResNet-50 + UPerNet  & \res{6.72}{0.35} & \best{7.69}{0.36}{0.97} & \res{10.25}{0.39} & \best{11.84}{0.51}{1.59} & \res{14.83}{0.57} & \best{15.43}{0.20}{0.60} & \textit{\textcolor{gray}{42.05}} \\
ResNet-101 + UPerNet & \res{7.40}{0.49} & \best{8.64}{0.49}{1.24} & \res{10.99}{0.32} & \best{13.14}{0.46}{2.15} & \res{15.89}{0.63} & \best{16.82}{0.28}{0.93} & \textit{\textcolor{gray}{43.82}} \\
\midrule
\multicolumn{8}{l}{\cellcolor{gray!10}\textit{\textbf{(b) Impact of Backbone Architecture}}} \\
Swin-S + UPerNet     & \res{8.64}{0.41} & \best{10.03}{0.47}{1.39} & \res{13.40}{0.57} & \best{15.86}{0.61}{2.46} & \res{19.48}{0.83} & \best{21.44}{0.36}{1.96} & \textit{\textcolor{gray}{47.72}} \\
ConvNeXt-S + UPerNet & \res{9.94}{0.43} & \best{12.64}{0.58}{2.70} & \res{15.36}{0.54} & \best{18.92}{0.46}{3.56} & \res{21.83}{0.44} & \best{23.80}{0.24}{1.97} & \textit{\textcolor{gray}{48.56}} \\
\midrule
\multicolumn{8}{l}{\cellcolor{gray!10}\textit{\textbf{(c) Impact of Semantic Segmentation Mechanism}}} \\
ResNet-50 + DeepLabV3 & \res{6.42}{0.38} & \best{7.25}{0.36}{0.83} & \res{10.06}{0.44} & \best{11.31}{0.50}{1.25} & \res{14.24}{0.56} & \best{14.74}{0.17}{0.50} & \textit{\textcolor{gray}{42.66}} \\
SegFormer-B5          & \res{12.05}{0.68} & \best{15.60}{0.59}{3.55} & \res{18.59}{0.47} & \best{21.52}{0.56}{2.93} & \res{24.10}{0.55} & \best{25.47}{0.24}{1.37} & \textit{\textcolor{gray}{49.13}} \\
\bottomrule
\end{tabular}%
}
\end{table}
}

\title{S$^2$-RL: Sample-Set Dual Reinforcement Learning for Generative Semantic Segmentation Dataset Distillation}

\author{%
  \textbf{Haoyu Wang\textsuperscript{1}} \quad
  \textbf{Fei Zhou\textsuperscript{2}} \quad
  \textbf{Qingqing Qiu\textsuperscript{3}} \quad
  \textbf{Lei Zhang\textsuperscript{1}}\thanks{Corresponding author.} \quad
  \textbf{Wei Wei\textsuperscript{1}} \quad
  \textbf{Chen Ding\textsuperscript{4}} \\
  \normalfont\textsuperscript{1}Northwestern Polytechnical University \quad
  \normalfont\textsuperscript{2}The Chinese University of Hong Kong \\
  \normalfont\textsuperscript{3}Zhejiang University \quad
  \normalfont\textsuperscript{4}Xi'an University of Posts \& Telecommunications \\
  \normalfont\url{https://github.com/00why00/S2RL}
}

\begin{document}

\setcounter{footnote}{1} %
\maketitle
\setcounter{footnote}{0}

\begin{abstract}
While dataset distillation has witnessed significant advances in image classification, its extension to dense prediction tasks such as semantic segmentation is plagued by two core bottlenecks: bi-level optimization-based approaches suffer from prohibitive computational costs stemming from pixel-wise gradient unrolling, whereas proxy-based generative methods merely infuse visual textures into fixed semantic masks derived from real data, thus failing to condense rich semantic knowledge into a compact set of informative, novel scene generations.
To address both bottlenecks simultaneously, we propose a novel Sample-Set Reinforcement Learning (S$^2$-RL) framework for generative semantic segmentation dataset distillation (SSDD). Specifically, S$^2$-RL formulates SSDD as a diffusion model based text-to-image generation paradigm, enabling flexible generation of images with informative, novel semantic distributions via a concise text prompt. Furthermore, we fine-tune a diffusion policy via Group Relative Policy Optimization (GRPO), leveraging a sample-set dual reward        
paradigm: for individual samples, it enforces semantic alignment with input text prompts and intra-class feature diversity; for sample sets, it maximizes inter-sample semantic distribution diversity by solving a contextual multi-armed bandit problem. These design choices enable S$^2$-RL to distill large-scale semantic segmentation datasets into a compact set of generated samples that encapsulate diverse, informative semantic knowledge without the need of pixel-wise gradient unrolling.
Extensive experiments on the ADE20K and COCO datasets demonstrate that S$^2$-RL achieves substantial and consistent improvements over state-of-the-art baselines, establishing a new benchmark for SSDD.
\end{abstract}

\begingroup
\setlength{\intextsep}{4pt}
\LockedMotivationFigure
\endgroup

\section{Introduction}

Data-efficient learning has become crucial as deep learning models grow increasingly data-hungry. Dataset distillation \cite{dd} addresses this by condensing massive datasets into compact synthetic surrogates. Existing advances in dataset distillation fall into two primary paradigms: optimization-based methods and proxy-based generative methods. While both paradigms yield strong performance on image classification, their extension to dense prediction tasks such as semantic segmentation remains fraught with substantial challenges. %

Specifically, optimization-based methods (Figure~\ref{fig:motivation_optimization}), such as gradient matching \cite{GM1,GM2,GM3} and distribution matching \cite{DM1,DM2,DM3}, formulate dataset distillation as a bi-level optimization problem %
aiming to align the gradients or distributions of student models on synthetic data with those on the real counterpart. While theoretically rigorous, these methods suffer from computational intractability in semantic segmentation tasks. Unlike image classification, which only relies on sparse supervision, semantic segmentation requires unrolling pixel-wise gradients at full spatial resolution, leading to severe memory overhead and making end-to-end optimization practically infeasible.

To circumvent this memory bottleneck, proxy-based generative methods~\cite{sre2l,d4m} (Figure~\ref{fig:motivation_proxy}) have emerged as an efficient alternative, which bypass bi-level optimization via data inversion or global statistic matching. %
Since global statistics matching cannot comprehensively depict the complex spatial arrangements of multiple categories, %
existing approaches, such as those for object detection \cite{ff} or universal dataset distillation \cite{ff2}, turn to rely on fixed spatial masks from real data for structural regularization. This strategy %
reduces the model to a mere texture filler within static boundaries, rather than a genuine scene composer, and thus fails to synthesize a compact set of novel, informative spatial layouts %
necessary for covering the long-tail distributions of real-world scenes.

To bridge these gaps, we propose a novel \textbf{S}ample-\textbf{S}et \textbf{R}einforcement \textbf{L}earning (S$^2$-RL) framework for generative semantic segmentation dataset distillation (SSDD), as shown in Figure~\ref{fig:motivation_ours}. Unlike existing gradient-heavy or mask-conditioned methods, we reformulate SSDD as a diffusion-based text-to-image generation paradigm and optimize it through reward-driven reinforcement learning without pixel-wise gradient unrolling through downstream segmentation training. By framing the generative model as an agent, S$^2$-RL directly utilizes appropriate rewards to drive policy updates. This framework not only eliminates the need for heavy second-order gradient computations, but also empowers the model to flexibly synthesize a compact set of images with novel, informative contents from concise text prompts, thus enabling effective SSDD.%

However, applying RL to generative dataset distillation introduces a new challenge. Standard RL methods (e.g., DDPO \cite{ddpo}, Flow-GRPO \cite{flow-grpo,grpo-guard}) are typically designed for sample-level optimization, aiming to maximize the quality of individual instances. In contrast, dataset distillation is fundamentally a set-level optimization problem: the objective is to construct a collective dataset that covers the original dataset distribution with fewer samples. Straightforwardly applying standard RL is prone to incur model collapse, where the model over-optimizes for easy patterns while ignoring the long-tail distribution characteristics of real datasets. To address this problem, we introduce a sample-set dual reward mechanism. At the sample-level, for each individual generated sample, we enforce semantic alignment with the input prompt while penalizing collapsed local textures within the same class. At the set-level, we maximize inter-sample semantic distribution diversity by solving a contextual multi-armed bandit problem, enabling the generated compact sample set to appropriately cover the original long-tail dataset distribution. %
Empirical evaluations on ADE20K \cite{ade20k} and COCO \cite{coco} benchmarks demonstrate that S$^2$-RL outperforms state-of-the-art baselines, reconciling generation quality with distributional diversity.

Our main contributions are summarized as follows:
\begin{itemize}[leftmargin=*,itemsep=0mm]	
    \item We propose S$^2$-RL, a novel generative dataset distillation framework that leverages RL to enable efficient image synthesis for dense prediction tasks, bypassing the memory bottlenecks of explicit gradient matching.
    \item We design a dual reward mechanism that balances sample quality and set diversity, mitigating semantic mode collapse.
    \item We obtain new state-of-the-art SSDD performance on ADE20K and COCO datasets.
\end{itemize}

\section{Related Work}

\subsection{Data-Efficient Learning}

\textbf{Coreset Selection.} This paradigm selects a representative subset from the original data. Classical methods include random selection, which samples data stochastically, and uniform selection, which ensures balanced sampling across categories. Beyond these, geometry-based methods leverage feature space structures. Herding \cite{herding1,herding2} selects samples closest to the cluster center of each class to approximate the mean of the distribution. Conversely, K-Center \cite{k-center1,k-center2} aims to minimize the maximum distance between any data point and the selected subset, effectively prioritizing boundary samples to maximize geometric coverage. While efficient, these extraction-based methods are inherently bounded by the finite samples available in the source domain, limiting their ability to bridge distributional gaps through synthesis.

\textbf{Optimization-based Dataset Distillation.} To synthesize informative data, early methods formulated dataset distillation as a bi-level optimization problem. Gradient matching methods \cite{GM1,GM2,GM3} align the gradients or long-term training dynamics of student and teacher models, while distribution matching methods \cite{DM1,DM2,DM3} optimize distribution or feature statistics. Differentiable augmentation~\cite{dsa}, neural feature regression~\cite{frepo}, and difficulty-aligned trajectory matching~\cite{datm} further advance classification distillation. Extending such methods to high-resolution dense prediction remains challenging. In particular, approaches that differentiate through student training face prohibitive memory costs when unrolling pixel-wise computation graphs.

\textbf{Proxy-based Generative Dataset Distillation.} To bypass the heavy optimization, recent works have shifted towards proxy-based generative methods. Model inversion methods \cite{sre2l,MI2,MI3} invert pretrained model statistics (e.g., BatchNorm) to recover training images. Generative attempts \cite{d4m,GDM2,GDM3} leverage the rich priors of generative models to synthesize diverse data distributions by matching global statistics or embedding alignments. Minimax Diffusion~\cite{minimax} additionally balances representativeness and diversity during generative training. However, a critical gap remains when transferring these successes to semantic segmentation. Due to the lack of spatial location information in global statistics, existing methods, such as those for object detection \cite{ff} or universal distillation \cite{ff2}, typically require pre-defined layouts from real data to constrain the generation process. This dependency limits the synthesized dataset to the topological variations already in the training set, hindering the exploration of novel scene layouts.

Unlike optimization methods that face computational barriers, or proxy methods that depend on static mask priors, S$^2$-RL adopts a reward-driven policy optimization paradigm. This allows for the autonomous construction of complex scene topologies without unrolling downstream student-training gradients or using fixed masks as generation conditions.

\subsection{Reinforcement Learning for Diffusion Models}

Reinforcement Learning has recently been integrated with diffusion models to align generation with specific objectives. DDPO \cite{ddpo} treats the multi-step denoising process as a policy, enabling the optimization of non-differentiable reward functions via policy gradient. DPO-Diffusion \cite{dpo} simplifies this by leveraging preference pairs to optimize the policy without an explicit reward model. To further enhance training stability, Flow-GRPO \cite{flow-grpo,grpo-guard} estimates gradients using the relative advantage of sampled groups, eliminating the need for a separate value network.

While these methods have successfully improved sample quality, they primarily focus on sample-level optimization. In the context of dataset distillation, our S$^2$-RL extends these frameworks to a set-level perspective, utilizing RL not just for individual image quality, but to actively explore and cover the diverse semantic distribution required for robust downstream training.

\section{Preliminaries}

\subsection{Problem Formulation}
\afterpage{\LockedMethodFigure}

\textbf{Semantic Segmentation Dataset Distillation.} Given a large-scale real dataset $\mathcal{D}_{real} = \{(x, y) | x \in \mathcal{X}, y \in \mathcal{Y}\}$, where $\mathcal{X} \subset \mathbb{R}^{H \times W \times 3}$ is the image and $\mathcal{Y} \subset \{0, 1\}^{H \times W \times C}$ is the pixel-wise label map for $C$ classes. Our goal is to learn a generative policy $\pi_\theta$ (parameterized by $\theta$) to synthesize a distilled dataset $\mathcal{D}_{syn}(\theta) = \{(\hat{x}_i, \hat{y}_i)\}_{i=1}^{M}$ with size $M \ll |\mathcal{D}_{real}|$. Ideally, a student segmentation network $f_{\phi}$ trained on $\mathcal{D}_{syn}$ should minimize the validation loss on $\mathcal{D}_{real}$. This is classically formulated as a bi-level optimization problem:
\begin{equation}
\begin{aligned}
\label{eq:bi-optim}
\min_{\theta} \quad & \mathbb{E}_{(x,y) \sim \mathcal{D}_{real}} [\mathcal{L}_{val}(f_{\phi^*(\theta)}(x), y)] \\
\text{s.t.} \quad & \phi^*(\theta) = \arg\min_{\phi} \mathbb{E}_{(\hat{x},\hat{y}) \sim \mathcal{D}_{syn}(\theta)} [\mathcal{L}_{train}(f_{\phi}(\hat{x}), \hat{y})],
\end{aligned}
\end{equation}
where $\mathcal{L}$ is the semantic segmentation loss. The policy determines $\mathcal{D}_{syn}(\theta)$ and thereby the trained student $\phi^*(\theta)$; $M$ is the number of retained image--label pairs.

\textbf{The Computational Bottleneck.} Optimization-based methods attempt to solve Eq.~\ref{eq:bi-optim} by computing the hypergradient $\nabla_{\theta} \mathcal{L}_{val}$ via chain rule. However, for dense prediction tasks, the inner-loop optimization requires backpropagating gradients through high-dimensional pixel-wise outputs. Unrolling this computation graph incurs a large memory cost $\mathcal{O}(H \cdot W \cdot C)$, rendering explicit bi-level optimization practically infeasible.

\subsection{Reinforcement Learning Reformulation}

To avoid differentiating through downstream student training, we reformulate dataset distillation as a reward-driven reinforcement learning problem. The pretrained diffusion model acts as the agent, and each text prompt $c$ provides semantic conditioning for a denoising trajectory. Its final image $\hat{x}$ is sampled from the policy $\pi_\theta(\cdot|c)$. The reward $R(\hat{x},c)$ provides a scalar proxy for the sample's utility in downstream segmentation. Consequently, the optimization objective shifts from minimizing validation loss to maximizing the expected reward:
\begin{equation}
\mathcal{J}_{RL}(\theta) = \mathbb{E}_{c \sim p(c), \hat{x} \sim \pi_{\theta}(\cdot|c)} [R(\hat{x}, c)].
\end{equation}

\subsection{Optimization via GRPO}

We optimize $\mathcal{J}_{RL}$ via Group Relative Policy Optimization (GRPO)~\cite{deepseekmath}, which uses group-relative advantages without a separate value network. For each prompt $c$, the rollout policy samples $G$ denoising trajectories, yielding final images $\{\hat{x}_i\}_{i=1}^{G}$. Their rewards are standardized within the prompt group to form advantages $\mathbb{A}_i$, which are reused across the trained denoising steps $\mathcal{T}$. Following diffusion-policy optimization~\cite{ddpo,flow-grpo,grpo-guard}, we maximize the clipped transition-wise surrogate
\begin{equation}
\label{eq:grpo}
\mathcal{J}_{GRPO}(\theta)=\mathbb{E}\!\left[
\frac{1}{G|\mathcal{T}|}\sum_{i=1}^{G}\sum_{t\in\mathcal{T}}w_t
\min\!\left(\widetilde{\rho}_{i,t}\mathbb{A}_i,
\operatorname{clip}(\widetilde{\rho}_{i,t},1-\epsilon,1+\epsilon)\mathbb{A}_i\right)
\right],
\end{equation}
where $\mathbb{A}_i$ is obtained by standardizing $R_i$ as $(R_i-\mu_c)/(\sigma_c+\varepsilon)$, followed by advantage clipping; $\mu_c$ and $\sigma_c$ are the reward mean and standard deviation within the prompt group. Here $\widetilde{\rho}_{i,t}$ is the RatioNorm-regulated current-to-rollout transition ratio, and $w_t$ is the inverse scheduler-step weight. Appendix~\ref{app:grpo_estimator} details the estimator, and Appendix~\ref{app:grpo_reward_details} provides the training settings.

\section{Method}

\subsection{Framework Overview}

As illustrated in Figure \ref{fig:method}, S$^2$-RL reformulates SSDD as an RL-based generation process. We fine-tune a pretrained Stable Diffusion policy $\pi_{\theta}$ to synthesize a dataset $\mathcal{D}_{syn}$ that maximizes downstream semantic segmentation performance. In each iteration, the agent samples a text prompt $c$ and synthesizes an image $\hat{x}$. To optimize this generative process without accessing prohibitive pixel-wise gradients, we employ GRPO to update the policy parameters $\theta$, directly maximizing a total reward signal $R_{total}$ based on the feedback from the synthesized samples.

The efficacy of this optimization relies on our proposed sample-set dual reward mechanism, which evaluates the generation at two distinct granularities. Class-list prompts specify category presence and co-occurrence while leaving spatial arrangements to the generator. In the following sections, we detail how $R_{sample}$ and $R_{set}$ are formulated to steer the policy toward high-fidelity and diverse synthesis.

\subsection{Sample-Level Reward}

The sample-level reward ($R_{sample}$) ensures that each generated sample $\hat{x}$ is not only semantically faithful to the conditioning prompt $c$ but also possesses discriminative pixel-level features suitable for training downstream semantic segmentation models. It consists of two components:

\textbf{Semantic Alignment (\bm{$R_{align}$}).} To ensure the generated image $\hat{x}$ faithfully reflects the target classes in prompt $c$, we compare the predicted masks against the concepts in prompt. Let $\mathcal{C}_{gt}$ be the set of target classes in prompt, and $\mathcal{C}_{pred}$ be the unique classes identified in the predicted mask $P_{pred} = f_{eval}(\hat{x})$, where $f_{eval}$ is the frozen segmentor. The semantic alignment reward is defined as:
\begin{equation}
\label{eq:align}
R_{align}(\hat{x},c)
=
(1-\lambda_{area})
\underbrace{
\frac{|\mathcal{C}_{gt} \cap \mathcal{C}_{pred}|}{|\mathcal{C}_{gt}|}
}_{\text{Concept Recall}}
+
\lambda_{area}
\underbrace{
\frac{1}{H W}
\sum_{k \in \mathcal{C}_{gt}}
\sum_{u=1}^{H}
\sum_{v=1}^{W}
\mathbb{I}\!\left(P_{pred}^{(u,v)} = k\right)
}_{\text{Area Occupancy}} ,
\end{equation}
where $P_{pred}^{(u,v)}$ denotes the predicted class at pixel $(u,v)$, $\lambda_{area}$ is a weighting factor, and $\mathbb{I}[\cdot]$ is the indicator function. The first term quantifies the recall of target categories, ensuring that the model does not miss classes in the text. The second term calculates the proportion of pixels belonging to the target classes, penalizing negligible or unrecognizable regions that fail to contribute valid gradients for training.

\textbf{Intra-class Diversity (\bm{$R_{div}$}).} Merely generating correct objects is insufficient. High-quality semantic segmentation data should also exhibit intra-class feature separability, reducing repetitive local textures for the same class. To address this, we introduce an intra-class diversity reward that operates on dense feature maps extracted by the frozen evaluator. For $n_k$ feature-grid locations assigned to class $k$ within one image, we uniformly sample $M_k=\min(n_k,100)$ vectors $\mathcal{F}_k=\{\mathbf{z}_{k,m}\}_{m=1}^{M_k}$. Let $\mathcal{C}^{\geq 2}_{pred} = \{k \in \mathcal{C}_{pred} : M_k \geq 2\}$. We define the diversity reward as
\begin{equation}
\label{eq:div}
R_{div}(\hat{x})
=
-\frac{1}{|\mathcal{C}^{\geq 2}_{pred}|}
\sum_{k \in \mathcal{C}^{\geq 2}_{pred}}
\frac{2}{M_k(M_k-1)}
\sum_{1 \leq m < n \leq M_k}
\left[
\cos(\mathbf{z}_{k,m}, \mathbf{z}_{k,n})-\delta
\right]_+ ,
\end{equation}
where $[x]_+ = \max(0, x)$ and $\delta$ is a similarity margin. Classes with $M_k<2$ are excluded; we set $R_{div}=0$ when $\mathcal{C}^{\geq 2}_{pred}=\emptyset$. Each frozen evaluator computes the reward using its own feature map and predicted mask, and the resulting scalar rewards are combined across evaluators. The penalty fires only when intra-class similarity exceeds $\delta$, pushing the generator toward semantically consistent yet texturally diverse features.

\subsection{Set-Level Reward}

While sample-level rewards improve individual images, they do not account for the semantic modes covered by other generated samples. To address this, we reformulate the set-level diversity optimization as a contextual multi-armed bandit (MAB) problem \cite{mab}, where the diffusion policy acts as the agent, the prompt $c$ serves as the context, and the arms are defined as the discrete collection of semantic prototypes specifically constructed for each category. Our objective is to encourage balanced arm visitation within each class, guiding the agent to cover diverse semantic modes rather than repeatedly exploiting frequent prototypes. To operationalize this discrete bandit framework within a continuous generative space, we introduce three key components: \textbf{Semantic Arm Construction} to discretize the feature space into class-specific prototypes; \textbf{Robust State-Action Mapping} to identify the pulled arms from generated images; and \textbf{Entropy-Maximizing Reward Mechanism} to incentivize the exploration of scarce prototypes.

\Needspace{34\baselineskip}
\afterpage{\LockedAdeTable\LockedTsneFigure}
\begin{wrapfigure}[27]{r}{0.54\textwidth}
\centering
\vspace{-10pt}
\begin{minipage}{\linewidth}
\captionsetup{type=algorithm,width=\linewidth,margin=0pt,singlelinecheck=false,justification=raggedright,skip=1pt}
\noindent\rule{\linewidth}{0.6pt}
\caption{Set-Level Reward}
\label{alg:sem}
\vspace{-5pt}
\noindent\rule{\linewidth}{0.4pt}
\small
\begin{algorithmic}[1]
\Statep{\textbf{Require:} Batch of images and prompt $\{(\hat{x}_i, c_i)\}_{i=1}^{B}$.}
\Statep{\textbf{Require:} Semantic arm set $\mathcal{P}$, global memory $N_{\mathrm{hist}}$.}
\Statep{\textbf{Ensure:} Reward batch $\mathbf{R}_{set} = \{r_i\}_{i=1}^{B}$.}
\Statep{\textbf{Initialize:} Batch accumulator $\Delta N_{\mathrm{batch}} \leftarrow 0$.}
\Statep{Detect classes and extract features for all $\hat{x}_i$.}
\Statex \textcolor{gray}{// Phase 1: Batch-Aware Online Calibration}
\For{each image $i = 1$ to $B$}
    \For{each detected target class $k$}
        \Statep{Match $\mathbf{z}_{i,k}$ to $\mathbf{p}^*_{i,k}$ via Eq.~\ref{eq:matching}.}
        \Statep{$s_{i,k} \leftarrow \max(0,\cos(\mathbf{z}_{i,k}, \mathbf{p}^*_{i,k}))$.}
        \State $\Delta N_{\mathrm{batch}}(\mathbf{p}^*_{i,k}) \leftarrow \Delta N_{\mathrm{batch}}(\mathbf{p}^*_{i,k}) + s_{i,k}$
    \EndFor
\EndFor
\Statex \textcolor{gray}{// Phase 2: Reward Evaluation}
\For{each image $i = 1$ to $B$}
    \State Initialize image accumulation score $S_i \leftarrow 0$
    \Statep{Get target object count $K_{obj} \leftarrow \operatorname{count}(c_i)$.}
    \For{each detected target class $k$}
        \Statep{$\hat{N}(\mathbf{p}) \leftarrow N_{\mathrm{hist}}(\mathbf{p}) + \Delta N_{\mathrm{batch}}(\mathbf{p})$.}
        \Statep{Evaluate $r(\mathbf{p}^*_{i,k})$ via Eq.~\ref{eq:sem}.}
        \State $S_i \leftarrow S_i + r(\mathbf{p}^*_{i,k})$
    \EndFor
    \State $r_i \leftarrow S_i / \max(1, K_{obj})$
\EndFor
\State \textbf{Global Update:} $N_{\mathrm{hist}} \leftarrow N_{\mathrm{hist}} + \Delta N_{\mathrm{batch}}$
\State \Return $\mathbf{R}_{set}$
\end{algorithmic}
\vspace{-5pt}
\noindent\rule{\linewidth}{0.6pt}
\end{minipage}
\vspace{-10pt}
\end{wrapfigure}

\textbf{Semantic Arm Construction.} To define a finite set of semantic arms, we perform an offline discretization of the real data manifold. We employ DINOv3~\cite{dinov3} to pool normalized features within real class regions using ground-truth training masks. For each category $k$, K-means with $K=50$ centers partitions the region features, and each center is mapped to its nearest observed same-class region feature. These representatives form the discrete arm set $\mathcal{P}_k$, serving as the reference for set-level coverage. If fewer than $K$ regions are available, all are retained. Appendix~\ref{app:prototype_construction} compares this construction with random same-class representatives.

\textbf{Robust State-Action Mapping.} When the agent generates an image $\hat{x}$, we must identify which specific arm has been pulled. This requires a robust mapping from pixel space to the prototype space. We use majority-vote masks from the frozen evaluator ensemble. We then perform spatially-masked pooling on the DINOv3 feature map of $\hat{x}_i$ and normalize the pooled feature to obtain a query vector $\mathbf{z}_{i,k}$ for category $k$. The pulled arm is identified as the nearest prototype in that class's bank:
\begin{equation}
\label{eq:matching}
\mathbf{p}^{*}_{i,k}=\arg\max_{\mathbf{p}\in\mathcal{P}_k}
\cos(\mathbf{z}_{i,k},\mathbf{p}),
\qquad
s_{i,k}=\max\!\left(0,\cos(\mathbf{z}_{i,k},\mathbf{p}^{*}_{i,k})\right).
\end{equation}

\textbf{Entropy-Maximizing Reward Mechanism.} To guide the agent towards underrepresented semantic modes, we design a dynamic reward based on soft visitation counts. These counts accumulate matching similarities for each prototype. For the current batch, we first compute $\Delta N_{\mathrm{batch}}$ and define the batch-aware state $\hat{N}=N_{\mathrm{hist}}+\Delta N_{\mathrm{batch}}$. Every image in the batch is evaluated against this shared state. For each requested class $k$ detected in $\hat{x}_i$, we compute the utility of its matched prototype $\mathbf{p}^{*}_{i,k}$:
\begin{equation}
\label{eq:sem}
r(\mathbf{p}^{*}_{i,k})
=
\underbrace{
s_{i,k}
\left(
\frac{\bar{N}_{k}+1}{\hat{N}(\mathbf{p}^{*}_{i,k})+1}
\right)^{\alpha}
}_{\text{Quality-Weighted Scarcity}}
+
\beta
\underbrace{
\sqrt{
\frac{
\ln\!\left(\max(N_{k},1)\right)
}{
\hat{N}(\mathbf{p}^{*}_{i,k})+1
}
}
}_{\text{Exploration Bonus}} ,
\end{equation}
where $s_{i,k}$ is the matching similarity, $\hat{N}(\mathbf{p}^{*}_{i,k})$ is the batch-aware soft count of the matched prototype, $N_k=\sum_{\mathbf{p}\in\mathcal{P}_k}\hat{N}(\mathbf{p})$ is the total batch-aware soft count of class $k$, and $\bar{N}_{k}=N_k/|\mathcal{P}_k|$ is the average soft count over prototypes of class $k$. The hyperparameters $\alpha$ and $\beta$ control scarcity reweighting and exploration strength, respectively. The image reward $R_{set}$ averages these utilities over all requested classes, with zero contribution from an absent requested class. Historical counts accumulate across policy-training batches and are updated after each full batch is scored, as outlined in Algorithm~\ref{alg:sem}.

\subsection{Total Objective and Theoretical Analysis}

To guide the policy update, we integrate the proposed hierarchical rewards into a unified feedback signal. The total reward $R_{total}$ for a generated image $\hat{x}$ conditioned on prompt $c$ is computed as a weighted summation:
\begin{equation}
R_{total}(\hat{x},c)
=
\gamma_{a} R_{align}(\hat{x},c)
+
\gamma_{d} R_{div}(\hat{x})
+
\gamma_{s} R_{set}(\hat{x},c),
\label{eq:total}
\end{equation}
where $\gamma_{a}, \gamma_{d}, \gamma_{s}$ are hyperparameters balancing semantic alignment, within-class diversity, and prototype coverage. We optimize the diffusion policy $\pi_\theta$ using Eq.~\ref{eq:grpo}.

\textbf{Optimization Behavior.} The combined reward objective biases the policy toward prompt-consistent, feature-diverse, and semantically well-covered samples, discouraging degenerate solutions by jointly leveraging sample-level quality rewards and a set-level coverage reward. We refer readers to Proposition~\ref{tm:all} in Appendix~\ref{app:theory} for a detailed interpretation.

\section{Experiments}

\subsection{Experimental Settings}

\textbf{Datasets and Metrics.} We evaluate our framework on two challenging benchmarks: ADE20K \cite{ade20k} and COCO \cite{coco}. For COCO, we utilize the panoptic annotations converted into semantic segmentation masks to ensure high-quality pixel-wise supervision. As one of the first systematic studies of dataset distillation specifically tailored to semantic segmentation, we establish a standardized evaluation protocol adapted from the image classification task's images per class (IPC) metric; specifically, considering the multi-class nature of segmentation images, we determine the data budget based on global compression ratios rather than fixed class counts. Accordingly, we set the ratios to 0.25\%, 0.5\%, 1.0\% for ADE20K and 0.625\textperthousand, 1.25\textperthousand, 2.5\textperthousand\ for the larger COCO dataset to align with standard distillation scales (e.g., 10/20/50 IPC). We report mean intersection over union (mIoU) and mean pixel accuracy (mAcc) as the primary metrics. For the main comparisons, we construct 5 dataset groups and train 5 student seeds per group, reporting mean and standard deviation over the 25 downstream runs.

\textbf{Baselines.} We compare S$^2$-RL with two categories of baselines: (1) coreset selection methods: We include Random selection, Uniform sampling (class-balanced), K-Center \cite{k-center1,k-center2}, and Herding \cite{herding1,herding2}. We employ a pretrained ResNet-50 \cite{resnet} to extract deep features for distance computation in K-Center and Herding. (2) generative baselines: Stable Diffusion 3.5 Medium (SD3.5)~\cite{sd35} serves as our vanilla generative baseline, and UniDD$^\ast$~\cite{ff2} denotes our reproduced UniDD-inspired proxy-guided segmentation baseline. Implementation details of UniDD$^\ast$ are provided in Appendix~\ref{app:unidd_baseline}.

\textbf{Implementation Details.} (1) RL and Set Construction: We employ Stable Diffusion 3.5 Medium \cite{sd35} as the policy backbone, fine-tuned via LoRA \cite{lora} with a rank of 32. To construct the text conditioning, we utilize the image subset selected by the Uniform baseline. Specifically, we extract all category names present in each selected image and concatenate them into a comma-separated string (e.g., ``person, car, tree'') to serve as the multi-category prompt. The policy is optimized using AdamW \cite{adamw} at learning rate 1e-4. After policy adaptation, we generate candidate images and screen them using the same sample-level and set-level reward functions used in training. A pretrained teacher model then supplies pixel-wise pseudo-labels for the retained images. (2) Evaluation Stage: We independently train a student on each synthesized dataset using MMSegmentation \cite{mmseg}. Our default student is SegFormer-B5 \cite{segformer}, whose MiT-B5 backbone is initialized from ImageNet-1K-pretrained weights for each run. The models are trained for 200 epochs using the AdamW optimizer with a crop size of $512 \times 512$. Additional architectural variations are detailed in the ablation studies.

\subsection{Main Results}

\LockedCocoTable
\LockedArchTable

\textbf{Quantitative Analysis.} Tables~\ref{tab:ade20k} and~\ref{tab:coco} present the results on ADE20K and COCO, where S$^2$-RL consistently achieves the best performance among the compared methods across all compression ratios. Compared with coreset selection strategies, the advantage is particularly evident at low data ratios, where limited image budgets make redundancy more costly. While selection draws from a fixed pool of real images, S$^2$-RL shapes the semantic content and coverage of new training examples through reward-guided synthesis. The DINOv3-based coreset comparison in Appendix~\ref{app:dinov3_coreset} provides a complementary evaluation with stronger selection features. S$^2$-RL also demonstrates a clear advantage over the generative baseline. This highlights the role of RL-driven alignment: while SD3.5 possesses strong general visual priors, S$^2$-RL steers generation toward informative supervision for downstream semantic segmentation. Additional controls appear in Appendix~\ref{app:matched_controls}.

\Needspace{16\baselineskip}
\begin{wraptable}[11]{r}{0.56\textwidth}
\captionsetup{font=small}
\vspace{-8pt}
\centering
\small
\setlength{\tabcolsep}{0pt}
\caption{Semantic/layout diagnostics on ADE20K 1.0\%. C/P: coverage/precision; R/P: recall/precision. Values are means over five stored sets; metric definitions and SDs are in Appendix~\ref{app:semantic_audit}.}
\label{tab:semantic_coverage}
\begin{tabular}{@{}p{0.18\linewidth}>{\centering\arraybackslash}p{0.30\linewidth}>{\centering\arraybackslash}p{0.23\linewidth}>{\centering\arraybackslash}p{0.29\linewidth}@{}}
\toprule
\textbf{Method} & \textbf{Occupancy C/P} & \textbf{Pair R/P} & \textbf{Geometry C/P} \\
\midrule
Uniform & .506/.929 & .592/.932 & .544/.916 \\
SD3.5 & .434/.888 & .539/.941 & .549/.880 \\
UniDD$^\ast$ & .427/.901 & .500/.955 & .464/.862 \\
\rowcolor{LightCyan} S$^2$-RL & .538/.929 & .665/.898 & .693/.939 \\
\bottomrule
\end{tabular}
\end{wraptable}

\textbf{Qualitative Analysis.} We visualize the latent feature distributions in Figure~\ref{fig:tsne}. A pretrained ResNet-50 extracts class-wise embeddings for \textit{Real Data} (\protect\graydot), the \textit{Uniform Selection Baseline} (\protect\bluetriangle), and \textit{S$^2$-RL} (\protect\redstar). In these projections, Uniform samples concentrate around densely populated regions, whereas S$^2$-RL samples occupy broader regions of the class-wise feature space. This pattern is consistent with the set-level reward, which favors underrepresented semantic prototypes alongside sample-level quality. For a compact segmentation dataset, covering a wider range of class-specific appearances can provide more varied supervision within the same image budget.

\textbf{Semantic and Layout Coverage.} To complement the feature-space visualization, Table~\ref{tab:semantic_coverage} evaluates spatial occupancy, category co-occurrence, and relative geometry using a frozen OneFormer~\cite{oneformer} outside the reward pipeline. Occupancy describes where each class appears, category pairs capture which classes occur together, and relative geometry describes their spatial relationships. The same annotator is applied to generated images and held-out real references. S$^2$-RL broadens occupancy and relative-geometry coverage while increasing category-pair recall. Images containing the same categories can still differ in object placement and relative scale. Covering these configurations exposes the student to varied spatial contexts within a fixed image budget. The category-pair results also reveal a coverage--precision trade-off as the range of generated combinations expands. These complementary measurements connect the distributional view to semantic content and scene organization relevant to dense prediction. Definitions and complete results appear in Appendix~\ref{app:semantic_audit}.

\Needspace{12\baselineskip}
\subsection{Ablation Study}

\begin{wraptable}[9]{r}{0.55\linewidth}
\renewcommand{\arraystretch}{0.95}
\captionsetup{font=small}
\setlength{\tabcolsep}{2.5pt}
\vspace{-10pt}
\caption{\textbf{Component effectiveness.} ADE20K performance across distillation ratios.}
\label{tab:ablation-component}
\centering
\small
\begin{tabular}{@{}lccc@{}}
\toprule
\textbf{Method Configuration} & \textbf{0.25\%} & \textbf{0.5\%} & \textbf{1.0\%} \\
\midrule
SD3.5 Baseline & \res{13.54}{0.56} & \res{18.09}{0.90} & \res{21.66}{0.40} \\
+ Alignment ($R_{align}$) & \res{14.82}{0.43} & \res{19.15}{0.57} & \res{22.91}{0.23} \\
+ Diversity ($R_{div}$)   & \res{15.17}{0.61} & \res{20.83}{0.42} & \res{24.73}{0.28} \\
\rowcolor{LightCyan} + Set ($R_{set}$) & \textbf{\res{15.60}{0.59}} & \textbf{\res{21.52}{0.56}} & \textbf{\res{25.47}{0.24}} \\
\bottomrule
\end{tabular}
\end{wraptable}

\textbf{Component Effectiveness Analysis.} To evaluate the contribution of each design element, we incrementally add: 1) Alignment ($R_{align}$), 2) Diversity ($R_{div}$), and 3) Set ($R_{set}$). Table~\ref{tab:ablation-component} shows consistent gains from combining these components across the evaluated data ratios. While $R_{align}$ establishes the semantic foundation for prompt adherence, $R_{div}$ encourages feature variation within class regions, reducing local redundancy. The set reward extends this optimization across images by favoring prototypes underrepresented in the generation history. The results highlight the complementary roles of local sample quality and cross-sample coverage. In a highly compressed dataset, this complementarity helps each image provide useful local supervision while contributing to semantic modes that other images cover less well.

\begin{figure}[t]
\captionsetup{font=small}
\centering
\includegraphics[width=0.9\linewidth]{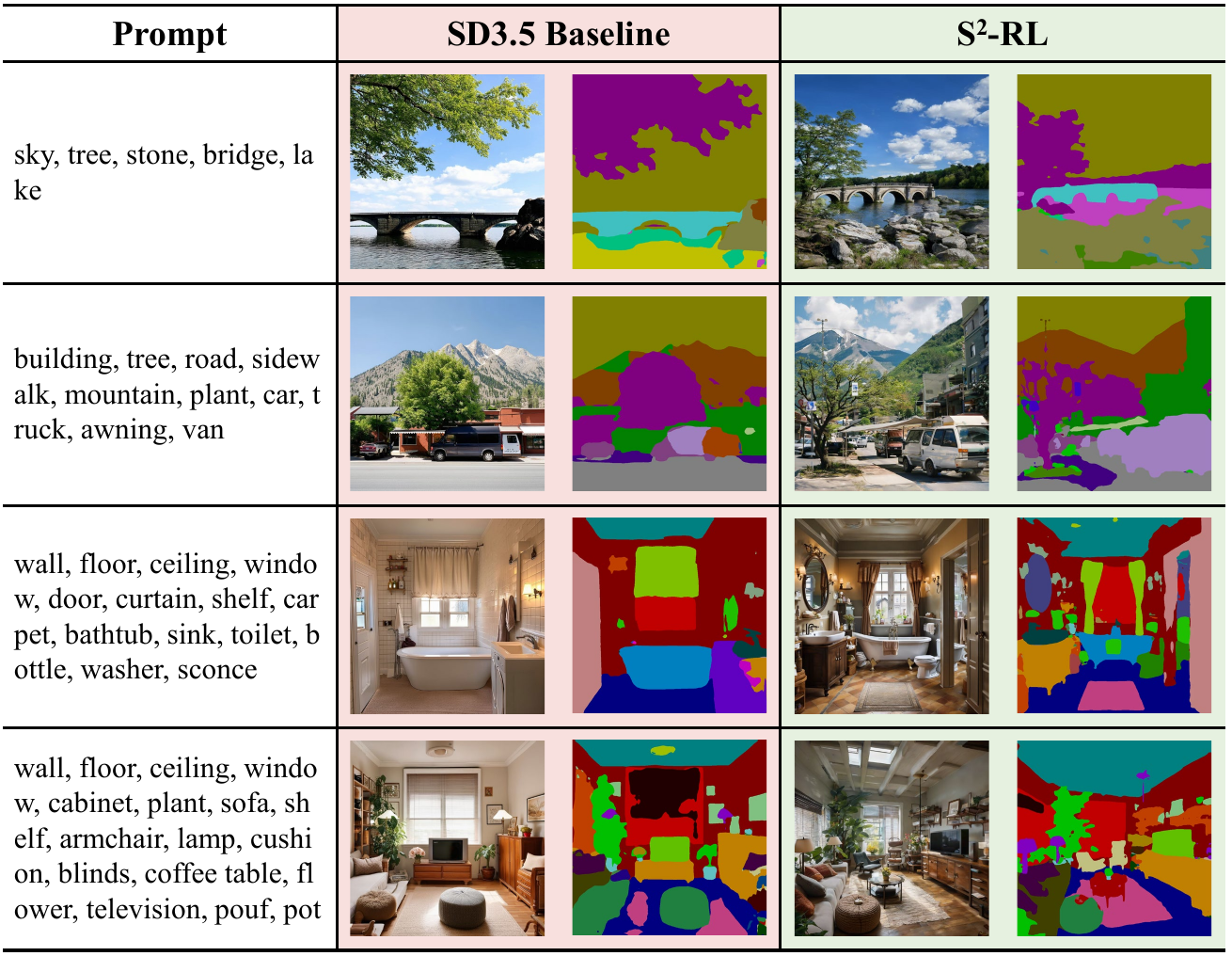}
\caption{Qualitative comparison of complex scene generation. We visualize images and pseudo-labels generated from multi-category prompts. In these examples, the SD3.5 baseline produces simplified scenes that omit some requested object categories, whereas S$^2$-RL synthesizes semantically richer images with more complex scene structure.}
\label{fig:ablation}
\end{figure}

\textbf{Robustness and Generalization.} Table~\ref{tab:ablation-arch} demonstrates the versatility of S$^2$-RL across three dimensions. First, regarding model capacity (Group a), gains persist from ResNet-18 to ResNet-101, showing that the compact dataset supports learning across model scales. Second, regarding backbone architectures (Group b), the benefits extend to both ConvNeXt~\cite{convnext} and Swin Transformer~\cite{swin}. Third, regarding segmentation mechanisms (Group c), improvements also appear with DeepLabV3~\cite{deeplabv3} and SegFormer~\cite{segformer}, indicating compatibility with different decoding designs. Together, these results show that the same distilled image--mask pairs provide useful supervision across the tested student configurations, allowing a compact dataset to be reused for different model requirements. Additional analyses of independent policy-training variability and prototype construction are provided in Appendices~\ref{app:policy_stability} and~\ref{app:prototype_construction}, respectively.

\textbf{Semantic Richness.} We assess data quality with the complex-scene examples in Figure~\ref{fig:ablation}. While the SD3.5 baseline generates texturally plausible images, it often omits objects requested in dense prompts, such as specific furniture in Row 4 or distinct vehicles in Row 2. In contrast, S$^2$-RL synthesizes scenes with higher semantic density, retaining richer structure and detail, from the foliage in Row 1 to the cluttered interior decor in Row 4. Such multi-category scenes provide supervision for several semantic regions and their shared context within each training image. This is valuable for dense prediction, where the student learns object appearance together with its surrounding scene structure. Preserving requested objects also retains boundaries between semantic regions, providing supervision for class transitions alongside region appearance. The qualitative examples complement the category co-occurrence and relative-geometry measurements in Table~\ref{tab:semantic_coverage}.

\section{Conclusion and Future Work}

We propose S$^2$-RL, a generative framework for semantic segmentation dataset distillation that fine-tunes Stable Diffusion via GRPO with a sample-set dual reward, integrating sample-level fidelity and set-level coverage to bypass the computational bottlenecks of bi-level optimization and the static-prior limitation of proxy-based methods. Experiments on ADE20K and COCO show consistent state-of-the-art gains, and our RL-driven paradigm offers a scalable path toward distillation for broader dense-prediction tasks. The framework incurs an upfront policy-adaptation cost and depends on pretrained generative and segmentation priors. Its rewards approximate downstream utility, and pseudo-labels may inherit teacher errors on rare or ambiguous categories. Future work could reduce adaptation and sampling costs, improve pseudo-label reliability through uncertainty-aware annotation, and incorporate lightweight student feedback to better align reward optimization with segmentation performance without full gradient unrolling.

\clearpage
\renewcommand{\acksection}{\section*{Acknowledgments}}
\begin{ack}
This work is supported in part by the National Natural Science Foundation of China under Grant 62372379, Grant 62472359, and Grant 62472350; in part by the Xi'an's Key Industrial Chain Core Technology Breakthrough Project: AI Core Technology Breakthrough under Grant 23ZDCYJSGG0003-2023; in part by Innovation Foundation for Doctor Dissertation of Northwestern Polytechnical University under Grant CX2025092.
\end{ack}

{\footnotesize
\bibliographystyle{plainnat}
\bibliography{references}

@article{dd,
  author       = {Tongzhou Wang and
                  Jun{-}Yan Zhu and
                  Antonio Torralba and
                  Alexei A. Efros},
  title        = {Dataset Distillation},
  journal      = {CoRR},
  volume       = {abs/1811.10959},
  year         = {2018},
}

@inproceedings{GM1,
  author       = {Bo Zhao and
                  Konda Reddy Mopuri and
                  Hakan Bilen},
  title        = {Dataset Condensation with Gradient Matching},
  booktitle    = {9th International Conference on Learning Representations, {ICLR} 2021,
                  Virtual Event, Austria, May 3-7, 2021},
  year         = {2021},
}

@inproceedings{GM2,
  author       = {George Cazenavette and
                  Tongzhou Wang and
                  Antonio Torralba and
                  Alexei A. Efros and
                  Jun{-}Yan Zhu},
  title        = {Dataset Distillation by Matching Training Trajectories},
  booktitle    = {{IEEE/CVF} Conference on Computer Vision and Pattern Recognition,
                  {CVPR} 2022, New Orleans, LA, USA, June 18-24, 2022},
  pages        = {10708--10717},
  publisher    = {{IEEE}},
  year         = {2022},
}

@inproceedings{GM3,
  author       = {Wenliang Zhong and
                  Haoyu Tang and
                  Qinghai Zheng and
                  Mingzhu Xu and
                  Yupeng Hu and
                  Weili Guan},
  title        = {Towards Stable and Storage-efficient Dataset Distillation: Matching
                  Convexified Trajectory},
  booktitle    = {{IEEE/CVF} Conference on Computer Vision and Pattern Recognition,
                  {CVPR} 2025, Nashville, TN, USA, June 11-15, 2025},
  pages        = {25581--25589},
  publisher    = {Computer Vision Foundation / {IEEE}},
  year         = {2025},
}

@inproceedings{DM1,
  author       = {Kai Wang and
                  Bo Zhao and
                  Xiangyu Peng and
                  Zheng Zhu and
                  Shuo Yang and
                  Shuo Wang and
                  Guan Huang and
                  Hakan Bilen and
                  Xinchao Wang and
                  Yang You},
  title        = {{CAFE}: Learning to Condense Dataset by Aligning Features},
  booktitle    = {{IEEE/CVF} Conference on Computer Vision and Pattern Recognition,
                  {CVPR} 2022, New Orleans, LA, USA, June 18-24, 2022},
  pages        = {12186--12195},
  publisher    = {{IEEE}},
  year         = {2022},
}

@inproceedings{DM2,
  author       = {Ganlong Zhao and
                  Guanbin Li and
                  Yipeng Qin and
                  Yizhou Yu},
  title        = {Improved Distribution Matching for Dataset Condensation},
  booktitle    = {{IEEE/CVF} Conference on Computer Vision and Pattern Recognition,
                  {CVPR} 2023, Vancouver, BC, Canada, June 17-24, 2023},
  pages        = {7856--7865},
  publisher    = {{IEEE}},
  year         = {2023},
}

@inproceedings{DM3,
  author       = {Fengli Ran and
                  Xiao Pu and
                  Bo Liu and
                  Xiuli Bi and
                  Bin Xiao},
  title        = {{TGDD}: Trajectory Guided Dataset Distillation with Balanced                              Distribution},
  booktitle    = {Fortieth {AAAI} Conference on Artificial Intelligence, {AAAI}
                  2026},
  publisher    = {{AAAI} Press},
  year         = {2026},
}

@inproceedings{sre2l,
  author       = {Zeyuan Yin and
                  Eric P. Xing and
                  Zhiqiang Shen},
  editor       = {Alice Oh and
                  Tristan Naumann and
                  Amir Globerson and
                  Kate Saenko and
                  Moritz Hardt and
                  Sergey Levine},
  title        = {Squeeze, Recover and Relabel: Dataset Condensation at ImageNet Scale
                  from a New Perspective},
  booktitle    = {Advances in Neural Information Processing Systems 36: Annual Conference
                  on Neural Information Processing Systems 2023, NeurIPS 2023, New Orleans,
                  LA, USA, December 10 - 16, 2023},
  year         = {2023},
}

@inproceedings{d4m,
  author       = {Duo Su and
                  Junjie Hou and
                  Weizhi Gao and
                  Yingjie Tian and
                  Bowen Tang},
  title        = {D\({}^{\mbox{4}}\)M: Dataset Distillation via Disentangled Diffusion
                  Model},
  booktitle    = {{IEEE/CVF} Conference on Computer Vision and Pattern Recognition,
                  {CVPR} 2024, Seattle, WA, USA, June 16-22, 2024},
  pages        = {5809--5818},
  publisher    = {{IEEE}},
  year         = {2024},
}

@inproceedings{ff,
  author       = {Ding Qi and
                  Jian Li and
                  Jinlong Peng and
                  Bo Zhao and
                  Shuguang Dou and
                  Jialin Li and
                  Jiangning Zhang and
                  Yabiao Wang and
                  Chengjie Wang and
                  Cairong Zhao},
  editor       = {Amir Globerson and
                  Lester Mackey and
                  Danielle Belgrave and
                  Angela Fan and
                  Ulrich Paquet and
                  Jakub M. Tomczak and
                  Cheng Zhang},
  title        = {Fetch and Forge: Efficient Dataset Condensation for Object Detection},
  booktitle    = {Advances in Neural Information Processing Systems 38: Annual Conference
                  on Neural Information Processing Systems 2024, NeurIPS 2024, Vancouver,
                  BC, Canada, December 10 - 15, 2024},
  year         = {2024},
}

@inproceedings{ff2,
  author       = {Ding Qi and
                  Jian Li and
                  Junyao Gao and
                  Shuguang Dou and
                  Ying Tai and
                  Jianlong Hu and
                  Bo Zhao and
                  Yabiao Wang and
                  Chengjie Wang and
                  Cairong Zhao},
  title        = {Towards Universal Dataset Distillation via Task-Driven Diffusion},
  booktitle    = {{IEEE/CVF} Conference on Computer Vision and Pattern Recognition,
                  {CVPR} 2025, Nashville, TN, USA, June 11-15, 2025},
  pages        = {10557--10566},
  publisher    = {Computer Vision Foundation / {IEEE}},
  year         = {2025},
}

@inproceedings{ddpo,
  author       = {Kevin Black and
                  Michael Janner and
                  Yilun Du and
                  Ilya Kostrikov and
                  Sergey Levine},
  title        = {Training Diffusion Models with Reinforcement Learning},
  booktitle    = {The Twelfth International Conference on Learning Representations,
                  {ICLR} 2024, Vienna, Austria, May 7-11, 2024},
  publisher    = {OpenReview.net},
  year         = {2024},
}

@article{flow-grpo,
  author       = {Jie Liu and
                  Gongye Liu and
                  Jiajun Liang and
                  Yangguang Li and
                  Jiaheng Liu and
                  Xintao Wang and
                  Pengfei Wan and
                  Di Zhang and
                  Wanli Ouyang},
  title        = {{Flow-GRPO}: Training Flow Matching Models via Online {RL}},
  journal      = {CoRR},
  volume       = {abs/2505.05470},
  year         = {2025},
}

@article{grpo-guard,
  author       = {Jing Wang and
                  Jiajun Liang and
                  Jie Liu and
                  Henglin Liu and
                  Gongye Liu and
                  Jun Zheng and
                  Wanyuan Pang and
                  Ao Ma and
                  Zhenyu Xie and
                  Xintao Wang and
                  Meng Wang and
                  Pengfei Wan and
                  Xiaodan Liang},
  title        = {{GRPO-Guard}: Mitigating Implicit Over-Optimization in Flow Matching
                  via Regulated Clipping},
  journal      = {CoRR},
  volume       = {abs/2510.22319},
  year         = {2025},
}

@article{mab,
  author       = {Aleksandrs Slivkins},
  title        = {Introduction to Multi-Armed Bandits},
  journal      = {Found. Trends Mach. Learn.},
  volume       = {12},
  number       = {1-2},
  pages        = {1--286},
  year         = {2019},
}

@inproceedings{ade20k,
  author       = {Bolei Zhou and
                  Hang Zhao and
                  Xavier Puig and
                  Sanja Fidler and
                  Adela Barriuso and
                  Antonio Torralba},
  title        = {Scene Parsing through {ADE20K} Dataset},
  booktitle    = {2017 {IEEE} Conference on Computer Vision and Pattern Recognition,
                  {CVPR} 2017, Honolulu, HI, USA, July 21-26, 2017},
  pages        = {5122--5130},
  publisher    = {{IEEE} Computer Society},
  year         = {2017},
}

@inproceedings{coco,
  author       = {Tsung{-}Yi Lin and
                  Michael Maire and
                  Serge J. Belongie and
                  James Hays and
                  Pietro Perona and
                  Deva Ramanan and
                  Piotr Doll{\'{a}}r and
                  C. Lawrence Zitnick},
  editor       = {David J. Fleet and
                  Tom{\'{a}}s Pajdla and
                  Bernt Schiele and
                  Tinne Tuytelaars},
  title        = {Microsoft {COCO:} Common Objects in Context},
  booktitle    = {Computer Vision - {ECCV} 2014 - 13th European Conference, Zurich,
                  Switzerland, September 6-12, 2014, Proceedings, Part {V}},
  series       = {Lecture Notes in Computer Science},
  volume       = {8693},
  pages        = {740--755},
  publisher    = {Springer},
  year         = {2014},
}

@book{k-center1,
  title={Facility location: concepts, models, algorithms and case studies},
  author={Farahani, Reza Zanjirani and Hekmatfar, Masoud},
  year={2009},
  publisher={Springer Science \& Business Media}
}

@inproceedings{k-center2,
  author       = {Ozan Sener and
                  Silvio Savarese},
  title        = {Active Learning for Convolutional Neural Networks: {A} Core-Set Approach},
  booktitle    = {6th International Conference on Learning Representations, {ICLR} 2018,
                  Vancouver, BC, Canada, April 30 - May 3, 2018, Conference Track Proceedings},
  year         = {2018},
}

@inproceedings{herding1,
  author       = {Max Welling},
  editor       = {Andrea Pohoreckyj Danyluk and
                  L{\'{e}}on Bottou and
                  Michael L. Littman},
  title        = {Herding dynamical weights to learn},
  booktitle    = {Proceedings of the 26th Annual International Conference on Machine
                  Learning, {ICML} 2009, Montreal, Quebec, Canada, June 14-18, 2009},
  series       = {{ACM} International Conference Proceeding Series},
  volume       = {382},
  pages        = {1121--1128},
  publisher    = {{ACM}},
  year         = {2009},
}

@inproceedings{herding2,
  author       = {Yutian Chen and
                  Max Welling and
                  Alexander J. Smola},
  editor       = {Peter Gr{\"{u}}nwald and
                  Peter Spirtes},
  title        = {Super-Samples from Kernel Herding},
  booktitle    = {{UAI} 2010, Proceedings of the Twenty-Sixth Conference on Uncertainty
                  in Artificial Intelligence, Catalina Island, CA, USA, July 8-11, 2010},
  pages        = {109--116},
  publisher    = {{AUAI} Press},
  year         = {2010},
}

@inproceedings{MI2,
  author       = {Jiawei Du and
                  Xin Zhang and
                  Juncheng Hu and
                  Wenxin Huang and
                  Joey Tianyi Zhou},
  editor       = {Amir Globerson and
                  Lester Mackey and
                  Danielle Belgrave and
                  Angela Fan and
                  Ulrich Paquet and
                  Jakub M. Tomczak and
                  Cheng Zhang},
  title        = {Diversity-Driven Synthesis: Enhancing Dataset Distillation through
                  Directed Weight Adjustment},
  booktitle    = {Advances in Neural Information Processing Systems 38: Annual Conference
                  on Neural Information Processing Systems 2024, NeurIPS 2024, Vancouver,
                  BC, Canada, December 10 - 15, 2024},
  year         = {2024},
}

@inproceedings{MI3,
  author       = {Zhiqiang Shen and
                  Ammar Sherif and
                  Zeyuan Yin and
                  Shitong Shao},
  title        = {{DELT}: {A} Simple Diversity-driven {Early-Late} Training for Dataset
                  Distillation},
  booktitle    = {{IEEE/CVF} Conference on Computer Vision and Pattern Recognition,
                  {CVPR} 2025, Nashville, TN, USA, June 11-15, 2025},
  pages        = {4797--4806},
  publisher    = {Computer Vision Foundation / {IEEE}},
  year         = {2025},
}

@inproceedings{GDM2,
  author       = {Jeffrey A. Chan{-}Santiago and
                  Praveen Tirupattur and
                  Gaurav Kumar Nayak and
                  Gaowen Liu and
                  Mubarak Shah},
  title        = {{MGD3}: Mode-Guided Dataset Distillation using Diffusion Models},
  booktitle    = {Forty-second International Conference on Machine Learning, {ICML}
                  2025, Vancouver, BC, Canada, July 13-19, 2025},
  year         = {2025},
}

@article{GDM3,
  author       = {Yawen Zou and
                  Guang Li and
                  Duo Su and
                  Zi Wang and
                  Jun Yu and
                  Chao Zhang},
  title        = {Dataset Distillation via Vision-Language Category Prototype},
  journal      = {CoRR},
  volume       = {abs/2506.23580},
  year         = {2025},
}

@inproceedings{dpo,
  author       = {Bram Wallace and
                  Meihua Dang and
                  Rafael Rafailov and
                  Linqi Zhou and
                  Aaron Lou and
                  Senthil Purushwalkam and
                  Stefano Ermon and
                  Caiming Xiong and
                  Shafiq Joty and
                  Nikhil Naik},
  title        = {Diffusion Model Alignment Using Direct Preference Optimization},
  booktitle    = {{IEEE/CVF} Conference on Computer Vision and Pattern Recognition,
                  {CVPR} 2024, Seattle, WA, USA, June 16-22, 2024},
  pages        = {8228--8238},
  publisher    = {{IEEE}},
  year         = {2024},
}

@inproceedings{resnet,
  author       = {Kaiming He and
                  Xiangyu Zhang and
                  Shaoqing Ren and
                  Jian Sun},
  title        = {Deep Residual Learning for Image Recognition},
  booktitle    = {2016 {IEEE} Conference on Computer Vision and Pattern Recognition,
                  {CVPR} 2016, Las Vegas, NV, USA, June 27-30, 2016},
  pages        = {770--778},
  publisher    = {{IEEE} Computer Society},
  year         = {2016},
}

@inproceedings{sd35,
  author       = {Patrick Esser and
                  Sumith Kulal and
                  Andreas Blattmann and
                  Rahim Entezari and
                  Jonas M{\"{u}}ller and
                  Harry Saini and
                  Yam Levi and
                  Dominik Lorenz and
                  Axel Sauer and
                  Frederic Boesel and
                  Dustin Podell and
                  Tim Dockhorn and
                  Zion English and
                  Robin Rombach},
  title        = {Scaling Rectified Flow Transformers for High-Resolution Image Synthesis},
  booktitle    = {Forty-first International Conference on Machine Learning, {ICML} 2024,
                  Vienna, Austria, July 21-27, 2024},
  year         = {2024},
}

@inproceedings{lora,
  author       = {Edward J. Hu and
                  Yelong Shen and
                  Phillip Wallis and
                  Zeyuan Allen{-}Zhu and
                  Yuanzhi Li and
                  Shean Wang and
                  Lu Wang and
                  Weizhu Chen},
  title        = {LoRA: Low-Rank Adaptation of Large Language Models},
  booktitle    = {The Tenth International Conference on Learning Representations, {ICLR}
                  2022, Virtual Event, April 25-29, 2022},
  year         = {2022},
}

@inproceedings{adamw,
  author       = {Ilya Loshchilov and
                  Frank Hutter},
  title        = {Decoupled Weight Decay Regularization},
  booktitle    = {7th International Conference on Learning Representations, {ICLR} 2019,
                  New Orleans, LA, USA, May 6-9, 2019},
  year         = {2019},
}

@misc{mmseg,
  title={{MMSegmentation}: OpenMMLab Semantic Segmentation Toolbox and Benchmark},
  author={MMSegmentation Contributors},
  howpublished = {\url{https://github.com/open-mmlab/mmsegmentation}},
  year={2020}
}

@inproceedings{upernet,
  author       = {Tete Xiao and
                  Yingcheng Liu and
                  Bolei Zhou and
                  Yuning Jiang and
                  Jian Sun},
  editor       = {Vittorio Ferrari and
                  Martial Hebert and
                  Cristian Sminchisescu and
                  Yair Weiss},
  title        = {Unified Perceptual Parsing for Scene Understanding},
  booktitle    = {Computer Vision - {ECCV} 2018 - 15th European Conference, Munich,
                  Germany, September 8-14, 2018, Proceedings, Part {V}},
  series       = {Lecture Notes in Computer Science},
  volume       = {11209},
  pages        = {432--448},
  publisher    = {Springer},
  year         = {2018},
}

@inproceedings{segformer,
  author       = {Enze Xie and
                  Wenhai Wang and
                  Zhiding Yu and
                  Anima Anandkumar and
                  Jos{\'{e}} M. {\'{A}}lvarez and
                  Ping Luo},
  editor       = {Marc'Aurelio Ranzato and
                  Alina Beygelzimer and
                  Yann N. Dauphin and
                  Percy Liang and
                  Jennifer Wortman Vaughan},
  title        = {{SegFormer}: Simple and Efficient Design for Semantic Segmentation with
                  Transformers},
  booktitle    = {Advances in Neural Information Processing Systems 34: Annual Conference
                  on Neural Information Processing Systems 2021, NeurIPS 2021, December
                  6-14, 2021, virtual},
  pages        = {12077--12090},
  year         = {2021},
}

@article{deeplabv3,
  author       = {Liang{-}Chieh Chen and
                  George Papandreou and
                  Florian Schroff and
                  Hartwig Adam},
  title        = {Rethinking Atrous Convolution for Semantic Image Segmentation},
  journal      = {CoRR},
  volume       = {abs/1706.05587},
  year         = {2017},
}

@article{dinov3,
  author       = {Oriane Sim{\'{e}}oni and
                  Huy V. Vo and
                  Maximilian Seitzer and
                  Federico Baldassarre and
                  Maxime Oquab and
                  Cijo Jose and
                  Vasil Khalidov and
                  Marc Szafraniec and
                  Seung Eun Yi and
                  Micha{\"{e}}l Ramamonjisoa and
                  Francisco Massa and
                  Daniel Haziza and
                  Luca Wehrstedt and
                  Jianyuan Wang and
                  Timoth{\'{e}}e Darcet and
                  Th{\'{e}}o Moutakanni and
                  Leonel Sentana and
                  Claire Roberts and
                  Andrea Vedaldi and
                  Jamie Tolan and
                  John Brandt and
                  Camille Couprie and
                  Julien Mairal and
                  Herv{\'{e}} J{\'{e}}gou and
                  Patrick Labatut and
                  Piotr Bojanowski},
  title        = {{DINOv3}},
  journal      = {CoRR},
  volume       = {abs/2508.10104},
  year         = {2025},
}

@inproceedings{convnext,
  author       = {Zhuang Liu and
                  Hanzi Mao and
                  Chao{-}Yuan Wu and
                  Christoph Feichtenhofer and
                  Trevor Darrell and
                  Saining Xie},
  title        = {A ConvNet for the 2020s},
  booktitle    = {{IEEE/CVF} Conference on Computer Vision and Pattern Recognition,
                  {CVPR} 2022, New Orleans, LA, USA, June 18-24, 2022},
  pages        = {11966--11976},
  publisher    = {{IEEE}},
  year         = {2022},
}

@inproceedings{swin,
  author       = {Ze Liu and
                  Yutong Lin and
                  Yue Cao and
                  Han Hu and
                  Yixuan Wei and
                  Zheng Zhang and
                  Stephen Lin and
                  Baining Guo},
  title        = {Swin Transformer: Hierarchical Vision Transformer using Shifted Windows},
  booktitle    = {2021 {IEEE/CVF} International Conference on Computer Vision, {ICCV}
                  2021, Montreal, QC, Canada, October 10-17, 2021},
  pages        = {9992--10002},
  publisher    = {{IEEE}},
  year         = {2021},
}

@misc{diffusers,
  author = {Patrick von Platen and Suraj Patil and Anton Lozhkov and Pedro Cuenca and Nathan Lambert and Kashif Rasul and Mishig Davaadorj and Dhruv Nair and Sayak Paul and William Berman and Yiyi Xu and Steven Liu and Thomas Wolf},
  title = {Diffusers: State-of-the-art diffusion models},
  year = {2022},
  publisher = {GitHub},
  journal = {GitHub repository},
  howpublished = {\url{https://github.com/huggingface/diffusers}}
}

@inproceedings{dpt,
  author       = {Ren{\'{e}} Ranftl and
                  Alexey Bochkovskiy and
                  Vladlen Koltun},
  title        = {Vision Transformers for Dense Prediction},
  booktitle    = {2021 {IEEE/CVF} International Conference on Computer Vision, {ICCV}
                  2021, Montreal, QC, Canada, October 10-17, 2021},
  pages        = {12159--12168},
  publisher    = {{IEEE}},
  year         = {2021},
}

@inproceedings{mask2former,
  author       = {Bowen Cheng and
                  Ishan Misra and
                  Alexander G. Schwing and
                  Alexander Kirillov and
                  Rohit Girdhar},
  title        = {Masked-attention Mask Transformer for Universal Image Segmentation},
  booktitle    = {{IEEE/CVF} Conference on Computer Vision and Pattern Recognition,
                  {CVPR} 2022, New Orleans, LA, USA, June 18-24, 2022},
  pages        = {1280--1289},
  publisher    = {{IEEE}},
  year         = {2022},
}

@inproceedings{eomt,
  author       = {Tommie Kerssies and
                  Niccol{\`{o}} Cavagnero and
                  Alexander Hermans and
                  Narges Norouzi and
                  Giuseppe Averta and
                  Bastian Leibe and
                  Gijs Dubbelman and
                  Daan de Geus},
  title        = {Your ViT is Secretly an Image Segmentation Model},
  booktitle    = {{IEEE/CVF} Conference on Computer Vision and Pattern Recognition,
                  {CVPR} 2025, Nashville, TN, USA, June 11-15, 2025},
  pages        = {25303--25313},
  publisher    = {Computer Vision Foundation / {IEEE}},
  year         = {2025},
}

@inproceedings{freestylenet,
  author       = {Han Xue and
                  Zhiwu Huang and
                  Qianru Sun and
                  Li Song and
                  Wenjun Zhang},
  title        = {Freestyle Layout-to-Image Synthesis},
  booktitle    = {{IEEE/CVF} Conference on Computer Vision and Pattern Recognition,
                  {CVPR} 2023, Vancouver, BC, Canada, June 17-24, 2023},
  pages        = {14256--14266},
  publisher    = {{IEEE}},
  year         = {2023},
}

@article{deepseekmath,
  author = {Zhihong Shao and Peiyi Wang and Qihao Zhu and Runxin Xu and Junxiao Song and Xiao Bi and Haowei Zhang and Mingchuan Zhang and Y. K. Li and Y. Wu and Daya Guo},
  title = {{DeepSeekMath}: Pushing the Limits of Mathematical Reasoning in Open Language Models},
  journal = {arXiv preprint arXiv:2402.03300},
  year = {2024},
  url = {https://arxiv.org/abs/2402.03300},
}

@inproceedings{dsa,
  author = {Bo Zhao and Hakan Bilen},
  title = {Dataset Condensation with Differentiable Siamese Augmentation},
  booktitle = {Proceedings of the 38th International Conference on Machine Learning},
  volume = {139},
  pages = {12674--12685},
  publisher = {PMLR},
  year = {2021},
  url = {https://proceedings.mlr.press/v139/zhao21a.html},
}

@inproceedings{frepo,
  author = {Yongchao Zhou and Ehsan Nezhadarya and Jimmy Ba},
  title = {Dataset Distillation using Neural Feature Regression},
  booktitle = {Advances in Neural Information Processing Systems},
  volume = {35},
  year = {2022},
  url = {https://proceedings.neurips.cc/paper_files/paper/2022/hash/3fe2a777282299ecb4f9e7ebb531f0ab-Abstract-Conference.html},
}

@inproceedings{datm,
  author = {Ziyao Guo and Kai Wang and George Cazenavette and Hui Li and Kaipeng Zhang and Yang You},
  title = {Towards Lossless Dataset Distillation via Difficulty-Aligned Trajectory Matching},
  booktitle = {International Conference on Learning Representations},
  year = {2024},
  url = {https://proceedings.iclr.cc/paper_files/paper/2024/hash/fc740286fb4840cb56f99400817967da-Abstract-Conference.html},
}

@inproceedings{minimax,
  author = {Jianyang Gu and Saeed Vahidian and Vyacheslav Kungurtsev and Haonan Wang and Wei Jiang and Yang You and Yiran Chen},
  title = {Efficient Dataset Distillation via Minimax Diffusion},
  booktitle = {IEEE/CVF Conference on Computer Vision and Pattern Recognition},
  pages = {15793--15803},
  year = {2024},
  url = {https://openaccess.thecvf.com/content/CVPR2024/html/Gu_Efficient_Dataset_Distillation_via_Minimax_Diffusion_CVPR_2024_paper.html},
}

@inproceedings{oneformer,
  author = {Jitesh Jain and Jiachen Li and MangTik Chiu and Ali Hassani and Nikita Orlov and Humphrey Shi},
  title = {{OneFormer}: One Transformer to Rule Universal Image Segmentation},
  booktitle = {IEEE/CVF Conference on Computer Vision and Pattern Recognition},
  year = {2023},
  url = {https://arxiv.org/abs/2211.06220},
}

@inproceedings{fid,
  author = {Martin Heusel and Hubert Ramsauer and Thomas Unterthiner and Bernhard Nessler and Sepp Hochreiter},
  title = {{GANs} Trained by a Two Time-Scale Update Rule Converge to a Local Nash Equilibrium},
  booktitle = {Advances in Neural Information Processing Systems},
  volume = {30},
  year = {2017},
  url = {https://papers.nips.cc/paper/2017/hash/8a1d694707eb0fefe65871369074926d-Abstract.html},
}

@inproceedings{lpips,
  author = {Richard Zhang and Phillip Isola and Alexei A. Efros and Eli Shechtman and Oliver Wang},
  title = {The Unreasonable Effectiveness of Deep Features as a Perceptual Metric},
  booktitle = {IEEE/CVF Conference on Computer Vision and Pattern Recognition},
  year = {2018},
  url = {https://arxiv.org/abs/1801.03924},
}

@article{vendi,
  author = {Dan Friedman and Adji Bousso Dieng},
  title = {The {Vendi} Score: A Diversity Evaluation Metric for Machine Learning},
  journal = {Transactions on Machine Learning Research},
  year = {2023},
  url = {https://openreview.net/forum?id=g97OHbQyk1},
}

@inproceedings{kid,
  author = {Miko{\l}aj Bi{\'n}kowski and Danica J. Sutherland and Michael Arbel and Arthur Gretton},
  title = {Demystifying {MMD GANs}},
  booktitle = {International Conference on Learning Representations},
  year = {2018},
  url = {https://arxiv.org/abs/1801.01401},
}

@inproceedings{prdc,
  author = {Muhammad Ferjad Naeem and Seong Joon Oh and Youngjung Uh and Yunjey Choi and Jaejun Yoo},
  title = {Reliable Fidelity and Diversity Metrics for Generative Models},
  booktitle = {Proceedings of the 37th International Conference on Machine Learning},
  volume = {119},
  pages = {7176--7185},
  publisher = {PMLR},
  year = {2020},
  url = {https://proceedings.mlr.press/v119/naeem20a.html},
}
}

\newpage
\appendix
\onecolumn
\section{Theoretical Elaboration}
\label{app:theory}

In this section, we provide additional derivations and interpretations supporting our proposed framework, including the gradient estimation of GRPO and an optimization interpretation of the sample-set dual reward.

\subsection{Gradient Estimation of GRPO}
\label{app:grpo_estimator}

A generated image $\hat{x}=x_0$ is the endpoint of a denoising trajectory $\tau=(x_T,\ldots,x_0)$. Its trajectory log-probability decomposes into transition log-probabilities, up to the initial noise distribution, which is independent of $\theta$:
\begin{equation}
\nabla_\theta\log\pi_\theta(\tau\mid c)
=\sum_{t=1}^{T}\nabla_\theta\log\pi_\theta(x_{t-1}\mid x_t,c).
\end{equation}
This decomposition motivates assigning a terminal reward to the stored denoising transitions without evaluating the intractable marginal probability of the final image. Our implemented update is a clipped transition-wise GRPO surrogate, rather than an exact trajectory-level importance-sampling estimator.

For a prompt $c$, the rollout policy $\pi_{\theta_{old}}$ generates $G$ trajectories. We evaluate the total terminal rewards $R_i$ against the shared batch count state described in Algorithm~\ref{alg:sem}, then form
\begin{equation}
\mathbb{A}_i=\operatorname{clip}\!\left(
\frac{R_i-\mu_c}{\sigma_c+\varepsilon},-5,5\right),
\qquad
\mu_c=\frac{1}{G}\sum_{i=1}^{G}R_i,
\end{equation}
where $\sigma_c$ is the reward standard deviation within this same-prompt group. The reward, count state, and advantage are held fixed during the policy update. The raw importance ratio for a stored transition is
\begin{equation}
\rho_{i,t}(\theta)=\exp\!\left[
\log\pi_\theta(x_{i,t-1}\mid x_{i,t},c)
-\log\pi_{\theta_{old}}(x_{i,t-1}\mid x_{i,t},c)
\right].
\end{equation}
Following GRPO-Guard~\cite{grpo-guard}, RatioNorm regulates the transition log-ratios before exponentiation; we denote the resulting ratio by $\widetilde{\rho}_{i,t}$. Inverse scheduler-step weighting supplies $w_t$ to balance contributions across trained denoising timesteps. The resulting objective is Eq.~\ref{eq:grpo}: each transition's ratio is clipped separately and its contribution is averaged across the group and trained timesteps. In particular, the implementation does not multiply transition ratios into one trajectory ratio and then clip that product.

The clipping width is $\epsilon=2\times10^{-6}$. We use one inner optimization epoch per rollout, refresh the rollout policy, update only LoRA parameters, clip gradient norms to 1.0, and maintain an exponential moving average (EMA). The reference-policy KL coefficient is zero. Thus, zero reference KL does not remove the transition clipping or these other update constraints. Independent-policy variability is assessed separately in Appendix~\ref{app:policy_stability}.

\subsection{Optimization Interpretation of Sample-Set Dual Reward}

In this section, we analyze the optimization behavior encouraged by the sample-set dual reward. The policy parameters $\theta$ are updated via GRPO to maximize the total reward objective
$\mathcal{J}_{total} = \mathbb{E}_{\pi_\theta}[R_{total}]$.
Since the total reward is a linear combination of three components (Eq.~\ref{eq:total}), the policy-gradient update is driven by the standardized advantage of $R_{total}$, $\mathbb{A}_{total}=(R_{total}-\mu_{R_{total}})/(\sigma_{R_{total}}+\varepsilon)$:
\begin{equation}
    \nabla_\theta \mathcal{J} \approx \mathbb{E}_{\hat{x}, c} \left[
    \mathbb{A}_{total}\,
    \nabla_\theta \log \pi_\theta(\hat{x}|c)
    \right].
\end{equation}
Because standardization is non-linear, $\mathbb{A}_{total}$ does not exactly factorize into a weighted sum of per-reward standardized advantages; nonetheless, $R_{total}$ is linear in the three reward components, so an increase (decrease) in any single component, holding the others fixed, monotonically increases (decreases) $\mathbb{A}_{total}$ within the same group. We use this monotonicity to interpret the role of each component below.

\begin{proposition}[Alignment-induced Suppression of Invalid Samples]
\label{tm:align}
For samples whose semantic content is inconsistent with the prompt, the alignment reward assigns a lower relative advantage. Under a GRPO update, this decreases their relative likelihood compared with prompt-consistent samples from the same group.
\end{proposition}

\begin{proof}
Consider a generated sample $\hat{x}$ that fails to match the prompt $c$, such as a sample missing the target objects. In this case, the alignment reward $R_{align}(\hat{x})$ is lower than that of semantically valid alternatives generated for the same prompt; holding $R_{div}$ and $R_{set}$ fixed within the group, the total reward $R_{total}(\hat{x})$ is correspondingly smaller, and after group-wise standardization the sample receives a smaller (typically negative) advantage $\mathbb{A}_{total}(\hat{x})$.

The corresponding policy-gradient contribution is
\begin{equation}
    \Delta \theta \propto
    \mathbb{A}_{total}(\hat{x}) \nabla_\theta \log \pi_\theta(\hat{x}|c).
\end{equation}
When $\mathbb{A}_{total}(\hat{x}) < 0$, this update decreases the relative log-probability of producing $\hat{x}$ under prompt $c$. Therefore, the alignment reward biases the policy away from semantically invalid generations and toward prompt-consistent samples. This effect should be interpreted as a local optimization pressure induced by the reward, rather than as a guarantee of eliminating all invalid samples.
\end{proof}

\begin{proposition}[Diversity Reward Penalizes Feature Collapse]
\label{tm:div}
The diversity reward penalizes highly similar pixel-level features within the same semantic category, thereby discouraging repetitive or collapsed intra-class structures.
\end{proposition}

\begin{proof}
Let $\mathcal{F}_k=\{\mathbf{z}_{k,m}\}_{m=1}^{M_k}$ be the set of pixel-level feature vectors belonging to class $k$ in the generated image $\hat{x}$. For classes with at least two feature vectors, the diversity reward in Eq.~\ref{eq:div} can equivalently be written as $R_{div}(\hat{x}) = -\frac{1}{|\mathcal{C}^{\geq 2}_{pred}|}\sum_{k\in\mathcal{C}^{\geq 2}_{pred}} \mathcal{L}_{div,k}$, where the per-class penalty
\begin{equation}
    \mathcal{L}_{div,k} = \frac{2}{M_k(M_k-1)}\sum_{1\leq m<n\leq M_k}
    \max\!\left(0,\cos(\mathbf{z}_{k,m},\mathbf{z}_{k,n})-\delta\right)
\end{equation}
measures the average unordered pairwise cosine similarity exceeding the margin $\delta$. Classes with $M_k<2$ are excluded from the average, and $R_{div}$ is set to zero if no class remains. Maximizing $R_{div}$ therefore corresponds to minimizing each valid $\mathcal{L}_{div,k}$.

When the generated region for class $k$ collapses into repetitive textures or nearly uniform structures, the corresponding feature vectors become highly similar, yielding $\cos(\mathbf{z}_{k,m},\mathbf{z}_{k,n}) > \delta$ for many pairs. This increases $\mathcal{L}_{div,k}$ and lowers $R_{div}$.

The gradient of the per-class penalty with respect to a feature vector has the form
\begin{equation}
    \nabla_{\mathbf{z}_{k,m}} \mathcal{L}_{div,k} \propto
    \sum_{\substack{n: \cos(\mathbf{z}_{k,m},\mathbf{z}_{k,n}) > \delta}}
    \frac{\mathbf{z}_{k,n} - (\mathbf{z}_{k,m} \cdot \mathbf{z}_{k,n})\mathbf{z}_{k,m}}{\|\mathbf{z}_{k,m}\| \|\mathbf{z}_{k,n}\|},
\end{equation}
where the simplified form above follows the normalized-vector form of cosine similarity. In implementation, $R_{div}$ is computed with the cosine-similarity operator on intermediate feature maps from the frozen segmentation evaluators, which normalizes each feature vector internally. The summation is taken over feature pairs whose similarity exceeds the margin. Moving in the negative-gradient direction (i.e., performing gradient descent on $\mathcal{L}_{div,k}$) reduces excessive pairwise similarity among features in the same class. Consequently, the diversity reward discourages trivial intra-class feature collapse and encourages more discriminative local structures within generated semantic regions.
\end{proof}

\begin{proposition}[Scarcity Reward Induces Negative Feedback]
\label{tm:s}
The scarcity term in the set-level reward induces a negative feedback mechanism over semantic prototypes: under-covered prototypes receive larger rewards, while over-covered prototypes receive smaller rewards.
\end{proposition}

\begin{proof}
Fix a class $k$ and let $\hat{N}(\mathbf{p}, t)$ denote the accumulated soft count of prototype $\mathbf{p} \in \mathcal{P}_k$ at training step $t$. Let $\bar{N}_k(t) = \frac{1}{|\mathcal{P}_k|}\sum_{\mathbf{p} \in \mathcal{P}_k} \hat{N}(\mathbf{p}, t)$ denote the average soft count over the prototypes of class $k$. Ignoring the exploration bonus for the moment, the scarcity term in Eq.~\ref{eq:sem} scales as
\begin{equation}
    r(\mathbf{p}, t) \propto
    \left(
    \frac{\bar{N}_k(t)+1}{\hat{N}(\mathbf{p}, t)+1}
    \right)^\alpha.
\end{equation}
If prototype $\mathbf{p}$ is under-covered, i.e., $\hat{N}(\mathbf{p}, t) < \bar{N}_k(t)$, then $r(\mathbf{p}, t)$ is strictly larger than the reward of an average-covered prototype $r(\bar{N}_k(t), \cdot)$. When samples assigned to this prototype are generated, they receive higher relative advantages, increasing their likelihood under the policy update. This accelerates the accumulation of $\hat{N}(\mathbf{p}, t)$.

Conversely, if prototype $\mathbf{p}$ is over-covered, i.e., $\hat{N}(\mathbf{p}, t) > \bar{N}_k(t)$, then its scarcity reward becomes smaller. Samples assigned to this prototype receive lower relative advantages, reducing their relative preference in subsequent updates. This forms a closed-loop negative feedback mechanism that encourages more balanced prototype coverage over training.
\end{proof}

\begin{proposition}[Exploration Bonus for Under-visited Prototypes]
\label{tm:e}
The exploration bonus increases the reward of persistently under-visited prototypes, providing an additional incentive for semantic exploration.
\end{proposition}

\begin{proof}
Fix a class $k$ and let $N_k(t) = \sum_{\mathbf{p} \in \mathcal{P}_k} \hat{N}(\mathbf{p}, t)$ denote the total soft count over prototypes of class $k$ at training step $t$. The exploration bonus in Eq.~\ref{eq:sem} takes the form
\begin{equation}
    B(\mathbf{p}, t) \propto
    \sqrt{
    \frac{
    \ln\!\left(\max(N_k(t),1)\right)
    }{
    \hat{N}(\mathbf{p}, t)+1
    }
    } .
\end{equation}
For a prototype $\mathbf{p}$ with a small soft count $\hat{N}(\mathbf{p}, t)$, this term becomes large relative to prototypes that have already been frequently visited. Therefore, when a generated sample is assigned to an under-visited prototype, it receives an additional reward bonus and a larger relative advantage in the GRPO update.

This mechanism does not guarantee that every prototype will be sampled with nonzero probability at all times. Rather, it increases the reward of under-visited prototypes once they are reached, thereby encouraging continued exploration and reducing the tendency of the policy to concentrate only on already frequent semantic modes.
\end{proof}

\begin{proposition}[Combined Effect of Sample-Set Dual Reward]
\label{tm:all}
The combined reward objective biases the policy toward prompt-consistent, feature-diverse, and semantically well-covered samples. It discourages degenerate solutions by combining sample-level quality rewards with a set-level coverage reward.
\end{proposition}

\begin{proof}
Let $\mathcal{S}$ be the generation space and let $\mathcal{S}_{valid} \subset \mathcal{S}$ denote the subset of samples that are semantically consistent with the prompt. The total reward introduces three complementary optimization effects:

\begin{enumerate}[leftmargin=*,itemsep=0mm]
    \item \textbf{Semantic alignment.} By Proposition~\ref{tm:align}, prompt-inconsistent samples receive lower relative advantages, which reduces their relative likelihood compared with semantically valid alternatives generated for the same prompt.

    \item \textbf{Feature diversity.} By Proposition~\ref{tm:div}, repetitive intra-class feature structures receive lower diversity rewards, which discourages low-information textures and encourages more discriminative local structures.

    \item \textbf{Set-level coverage.} By Propositions~\ref{tm:s} and~\ref{tm:e}, under-covered semantic prototypes receive larger set-level rewards, creating a negative feedback mechanism that promotes broader prototype coverage over the distilled set.
\end{enumerate}

Therefore, the sample-level rewards improve the local quality of individual generated images, while the set-level reward encourages global semantic coverage across the generated set. The combined objective does not impose exact uniformity or guarantee perfect convergence, but it provides a principled optimization bias toward informative and diverse dense-prediction training samples.
\end{proof}

\section{Detailed Implementation Settings}

In this section, we provide the comprehensive hyperparameter settings and architectural details required to reproduce our experiments. Our framework is implemented using PyTorch and the Diffusers library \cite{diffusers} based on GRPO Guard \cite{grpo-guard}.

\subsection{Architecture Details}

\textbf{Policy Model.} We utilize Stable Diffusion 3.5 Medium ~\cite{sd35} as our backbone policy $\pi_\theta$. To ensure training efficiency and prevent catastrophic forgetting of the pretrained generative priors, we adopt low-rank adaptation (LoRA) ~\cite{lora} for fine-tuning. 
\begin{itemize}[leftmargin=*,itemsep=0mm]
    \item \textbf{Base Model:} Stable Diffusion 3.5 Medium (2.5B parameters).
    \item \textbf{LoRA Configuration:} We inject LoRA layers into the attention mechanisms of the MMDiT. The LoRA rank is set to 32, with an alpha value of 64.
    \item \textbf{Trainable Parameters:} Only the LoRA weights ($\sim$ 70 MB) are updated during the RL phase.
\end{itemize}

\textbf{Reward Models.} Our dual-reward mechanism relies on three types of frozen pretrained models, each playing a distinct role:
\begin{itemize}[leftmargin=*,itemsep=0mm]
    \item \textbf{Semantic Evaluator Ensemble.} For ADE20K we use UPerNet-ConvNeXt-Large~\cite{upernet,convnext}, SegFormer-B5~\cite{segformer}, and DPT-Large~\cite{dpt}; for COCO we use Mask2Former-Swin-Large~\cite{mask2former}, EOMT-Large~\cite{eomt}, and EOMT-Base. Each frozen evaluator computes alignment from its own mask and diversity from its own intermediate feature grid, using at most 100 locations per predicted class. We sum the resulting scalar sample rewards. Raw feature vectors from different architectures are never compared. Per-pixel majority voting combines the masks only for DINOv3 region pooling in the set-level reward.
    \item \textbf{Pseudo-Label Teacher.} The downstream student is trained on synthesized images paired with pseudo-labels produced by a single, stronger frozen segmentor: UPerNet-ConvNeXt-XLarge for ADE20K and Mask2Former-Swin-Large for COCO. The same teacher is used to label images from \emph{all} synthetic baselines (SD3.5, UniDD$^\ast$, and S$^2$-RL), so the annotation procedure and downstream training protocol are shared across methods.
    \item \textbf{Feature Extractor.} The set-level reward (Eq.~\ref{eq:sem}) is computed on dense patch features from a frozen DINOv3 ViT-L model~\cite{dinov3}. The semantic mask from the segmentor ensemble is downsampled by nearest-neighbor interpolation to match the DINOv3 patch grid, and DINOv3 features are L$_2$-normalized along the channel dimension before any cosine-similarity computation. DINOv3 is used only for semantic-arm construction and matching in $R_{set}$.
\end{itemize}

\subsection{UniDD-inspired Proxy-guided Baseline}
\label{app:unidd_baseline}

We implement a UniDD-inspired proxy-guided baseline as an additional generative comparison. UniDD~\cite{ff2} uses task-specific proxy models to extract task-relevant supervision and guide diffusion-based dataset distillation. Since the official UniDD implementation is not publicly available, we implement a segmentation-oriented proxy-guided baseline using a released pretrained FreestyleNet~\cite{freestylenet} checkpoint as the layout-to-image generator, while following the proxy-guided sampling and refinement principle of UniDD. This implementation corresponds to the UniDD$^\ast$ entries in Table~\ref{tab:ade20k} and Table~\ref{tab:coco}.

The proxy-guided baseline initializes generation from the masks selected by the Uniform baseline. This keeps the layout source and compression budget aligned with the real-data subset baselines, while replacing the corresponding real images with generated images. FreestyleNet takes dense semantic layouts as conditions and synthesizes images aligned with the input masks, making it suitable for adapting proxy-guided generation to semantic segmentation.

The proxy segmentation model is dataset-specific. For ADE20K, we use SegFormer-B0 as the proxy segmentation model. For COCO, we use MaskFormer-R50 as the proxy segmentation model. During generation, the proxy model scores the semantic consistency between the generated image and the input mask, and this score is used for refinement and filtering. After image generation, we use the same pseudo-labeling teacher as S$^2$-RL to obtain pixel-wise annotations, so that all synthetic baselines are evaluated with the same downstream training pipeline.

We follow the original UniDD hyperparameter setting as closely as possible. Each image is generated with 250 sampling steps. The filtering threshold is set to 0.75. A generated image is accepted if its proxy score exceeds this threshold. If the proxy score remains below the threshold for three consecutive attempts, we keep the highest-scoring image among the attempted samples. This fallback rule preserves a fixed dataset size for every compression ratio and prevents failed generations from changing the number of downstream training images.

We report this method as UniDD$^\ast$ to indicate that it is a segmentation-oriented adaptation of UniDD rather than the official implementation. Compared with the vanilla SD3.5 baseline, UniDD$^\ast$ uses mask-conditioned generation and task-specific proxy filtering. Compared with S$^2$-RL, it does not fine-tune the diffusion policy with sample-set rewards, but instead relies on iterative proxy-guided sampling from the pretrained FreestyleNet generator.

\subsection{Hyperparameters}

We list the detailed hyperparameters used for training S$^2$-RL in Table ~\ref{tab:hyperparams}. The optimization is conducted on a per-dataset basis (ADE20K and COCO).

\begin{table}[ht]
    \centering
    \setlength{\tabcolsep}{20pt}
    \caption{Detailed Hyperparameters for S$^2$-RL Training.}
    \label{tab:hyperparams}
    \begin{tabular}{l|cc}
    \toprule
    \textbf{Hyperparameter} & \textbf{ADE20K} & \textbf{COCO} \\
    \midrule
    \multicolumn{3}{l}{\textit{Optimization (GRPO)}} \\
    Optimizer & \multicolumn{2}{c}{AdamW \cite{adamw}} \\
    Learning Rate & \multicolumn{2}{c}{$1 \times 10^{-4}$} \\
    Batch Size (Prompts) & \multicolumn{2}{c}{16} \\
    Group Size ($G$) & 8 & 16 \\
    Ratio Clipping Width ($\epsilon$) & \multicolumn{2}{c}{$2\times10^{-6}$} \\
    Advantage Clipping & \multicolumn{2}{c}{$[-5,5]$} \\
    Inner Epochs per Rollout & \multicolumn{2}{c}{1} \\
    Gradient Clipping & \multicolumn{2}{c}{1.0} \\
    KL Penalty Coefficient & \multicolumn{2}{c}{0.0} \\
    EMA & \multicolumn{2}{c}{true} \\
    \midrule
    \multicolumn{3}{l}{\textit{Reward Configuration}} \\
    Recall vs. Occupancy Weight ($\lambda_{area}$) & \multicolumn{2}{c}{0.3} \\
    Diversity Margin ($\delta$) & \multicolumn{2}{c}{0.8} \\
    Scarcity Decay Factor ($\alpha$) & \multicolumn{2}{c}{1.0} \\
    Exploration Bonus Weight ($\beta$) & \multicolumn{2}{c}{0.5} \\
    \midrule
    \multicolumn{3}{l}{\textit{Total Objective Weights (Eq.~\ref{eq:total})}} \\
    Alignment Weight ($\gamma_{a}$) & 1.0 & 1.0 \\
    Diversity Weight ($\gamma_{d}$) & 5.0 & 10.0 \\
    Set Weight ($\gamma_{s}$) & 3.0 & 5.0 \\
    \midrule
    \multicolumn{3}{l}{\textit{Clustering (Set Reward)}} \\
    Number of Clusters ($K$) & \multicolumn{2}{c}{50} \\
    Clustering Algorithm & \multicolumn{2}{c}{K-Means (Euclidean distance)} \\
    \midrule
    \multicolumn{3}{l}{\textit{Generate Configuration}} \\
    Inference Scheduler & \multicolumn{2}{c}{FlowMatchEulerDiscreteScheduler} \\
    Classifier-Free Guidance & \multicolumn{2}{c}{7.0} \\
    Inference Steps & \multicolumn{2}{c}{50} \\
    \bottomrule
    \end{tabular}
\end{table}

\paragraph{Zero-KL setting.}
We use the zero-KL setting adopted in the GRPO-Guard~\cite{grpo-guard} training configuration, where the KL loss is not applied. GRPO-Guard stabilizes GRPO-style optimization for diffusion/flow models through regulated clipping, implemented by ratio normalization and gradient reweighting across denoising timesteps. These operations restore more balanced importance-ratio distributions and reduce timestep-wise gradient imbalance during policy optimization.

\subsection{GRPO Reward Computation Details}
\label{app:grpo_reward_details}

We describe the reward computation process used during GRPO fine-tuning. In each training iteration, we sample 16 text prompts from the prompt pool. For each prompt, the rollout policy generates a group of 8 trajectories for ADE20K or 16 for COCO, each producing one candidate image. Rewards are computed on the final images, while the policy update uses the stored denoising trajectories to evaluate the clipped per-step importance-ratio objective described in Appendix~\ref{app:grpo_estimator}.

For each generated image, we compute the sample-level reward terms, including semantic alignment and intra-class feature diversity. The semantic alignment reward evaluates whether the generated image is consistent with the target semantic content specified by the prompt. The intra-class diversity reward evaluates whether the image contains discriminative class-level features rather than repetitive or collapsed visual patterns. These two terms are computed independently for each image and provide local supervision for sample quality.

The set-level reward is computed on the same generated batch rather than on a separate reward batch. After extracting semantic features from all generated images in the current iteration, each feature is assigned to its nearest semantic prototype. We then accumulate the prototype-count increment induced by the current mini-batch. The set-level reward is evaluated using both the historical prototype counts and the current batch-level increment, before committing the new counts to the global memory. This batch-aware update prevents reward inflation \emph{across} iterations rather than within a single batch: samples that match the same prototype within a single iteration receive the same exploration bonus, but historical counts are updated only after the entire batch has been evaluated, so later iterations correctly see the marginal coverage contributed by the previous one.

After combining the sample-level and set-level rewards, GRPO computes advantages within each same-prompt group. Specifically, the $G$ rewards belonging to one prompt are normalized within that group, and the resulting advantages are clipped to $[-5,5]$ before use in the surrogate. We do not normalize rewards across all 16 prompts in the iteration. This design ensures that the policy compares alternative generations under the same semantic condition, rather than mixing reward scales across unrelated prompts.

The sample-level and set-level rewards serve different roles in optimization. The sample-level rewards guide each individual image toward semantic correctness and feature discriminability. The set-level reward tracks prototype visitation during policy training and assigns larger bonuses to under-covered semantic modes. Therefore, the set-level reward does not duplicate the sample-level objectives; instead, it introduces a dataset-level exploration signal that encourages the policy to construct a more balanced and informative synthetic set.

\subsection{Computational Resources}
\label{app:computational_resources}

Table~\ref{tab:construction_cost} summarizes the construction cost and peak GPU memory of different dataset construction pipelines on ADE20K at the 1.0\% budget. For selection-based baselines, the reported cost corresponds to feature extraction and subset selection. For generative baselines, the reported cost includes the generation or adaptation stage required by the corresponding pipeline; for UniDD$^\ast$, the FreestyleNet training cost is taken from the original FreestyleNet paper because our experiments use its released pretrained checkpoint.

\begin{table}[t]
\centering
\caption{Construction cost and peak GPU memory on ADE20K at the 1.0\% budget.}
\label{tab:construction_cost}
\begin{tabular}{lcc}
\toprule
Method & Construction Cost & Peak GPU Memory \\
\midrule
Random & selection $<$1 min & N/A \\
Uniform & selection $<$1 min & N/A \\
K-Center & selection $\sim$10 min & $\sim$4 GB \\
Herding & selection $\sim$10 min & $\sim$4 GB \\
SD3.5 Baseline & inference $\sim$1 GPU-hr & $\sim$40 GB \\
UniDD$^\ast$ & \makecell{FreestyleNet training $\sim$48 GPU-hrs \\ inference $\sim$10 GPU-hrs} & $\sim$30 GB \\
S$^2$-RL & \makecell{RL training $\sim$48 GPU-hrs \\ inference $\sim$1 GPU-hr} & $\sim$46 GB \\
\bottomrule
\end{tabular}
\end{table}

The selection-based baselines only operate on the original training set and therefore require substantially lower construction cost. Random and Uniform selection do not require feature extraction, while K-Center and Herding additionally compute selection features before subset construction.

The SD3.5 baseline directly samples images from the pretrained text-to-image model and does not include task-specific policy adaptation. UniDD$^\ast$ uses FreestyleNet for mask-conditioned generation and performs proxy-guided refinement during inference. The FreestyleNet training cost in Table~\ref{tab:construction_cost} is converted from the FreestyleNet supplementary material~\cite{freestylenet}, which reports approximately 2 days of training on ADE20K using a single NVIDIA A100 GPU, corresponding to approximately 48 GPU-hours. In our experiments, UniDD$^\ast$ starts from the released pretrained FreestyleNet checkpoint, and the measured inference cost corresponds to proxy-guided generation with this checkpoint.

S$^2$-RL uses pretrained SD3.5 as the initial policy and performs reward-guided RL fine-tuning on our RTX Pro 6000 GPUs. The reported S$^2$-RL cost includes the RL adaptation stage and the subsequent synthetic set generation stage. Downstream student training is conducted under the same protocol for all constructed training sets.

\subsection{Generator and Sampling-Budget Controls}
\label{app:matched_controls}

\begin{table}[htbp]
\captionsetup{font=small}
\centering
\small
\setlength{\tabcolsep}{3pt}
\caption{Additional controls on ADE20K 1.0\% with SegFormer-B5. Both blocks use SD3.5; costs in (b) are approximate generative GPU-hours. Scores are mean $\pm$ SD over 5 sets $\times$ 5 student seeds.}
\label{tab:matched_controls}
\begin{tabular}{lccc}
\toprule
Method & GPU-hours & mIoU (\%) & mAcc (\%) \\
\midrule
\multicolumn{4}{l}{\textit{(a) Shared generator backbone}} \\
UniDD$^\ast$ (SD3.5 control) & -- & \res{23.95}{0.41} & \res{33.80}{0.92} \\
S$^2$-RL & -- & \textbf{\res{25.47}{0.24}} & \textbf{\res{35.58}{0.57}} \\
\midrule
\multicolumn{4}{l}{\textit{(b) Approximately matched generative computation}} \\
SD3.5 best-of-$N$ & $\sim$50 & \res{23.29}{0.36} & \res{33.04}{0.83} \\
S$^2$-RL & $\sim$49 & \textbf{\res{25.47}{0.24}} & \textbf{\res{35.58}{0.57}} \\
\bottomrule
\end{tabular}
\end{table}

\textbf{Shared generator backbone.} For the ADE20K 1.0\% control in Table~\ref{tab:matched_controls}(a), we adapt SD3.5 to accept semantic masks following the FreestyleNet conditioning approach, train it on ADE20K, and apply the UniDD-inspired proxy-guided procedure. Both methods use SegFormer-B5 with five constructed sets and five student seeds per set. The original multi-budget UniDD$^\ast$ results in Tables~\ref{tab:ade20k} and~\ref{tab:coco} use the FreestyleNet-based implementation described in Appendix~\ref{app:unidd_baseline}.

\textbf{Approximately matched generative computation.} The unadapted SD3.5 best-of-$N$ control uses the same class-list prompts and generates approximately 200 candidates per prompt. Candidates are screened with the same sample-level and set-level reward functions used in policy training, retaining one image per prompt and 202 images per set. S$^2$-RL generates four candidates per prompt after policy adaptation and retains the same dataset size. Best-of-$N$ uses approximately 50 GPU-hours for generation; S$^2$-RL uses approximately 48 for adaptation and 1 for generation. These totals compare generative GPU-hours, without a complete accounting of preprocessing, reward evaluation, pseudo-labeling, or downstream training. Both methods follow the same five-set, five-student-seed evaluation protocol.

\textbf{Generator and Compute Controls.} Table~\ref{tab:matched_controls}(a) compares S$^2$-RL with a mask-conditioned UniDD$^\ast$ adaptation using the same SD3.5 backbone. The advantage persists with a shared generator, highlighting the value of coupling flexible text-conditioned generation with sample-set rewards for segmentation distillation. Table~\ref{tab:matched_controls}(b) compares policy adaptation with best-of-$N$ sampling under a similar generative budget. Both approaches use the same reward functions to screen candidates, and the adapted policy yields stronger downstream performance. This suggests a benefit from shaping the distribution from which candidates are drawn when constructing a compact training set.

\subsection{Downstream Training Details}
We train student models independently on each distilled dataset. Each SegFormer-B5 run initializes its MiT-B5 backbone from ImageNet-1K-pretrained weights.
\begin{itemize}[leftmargin=*,itemsep=0mm]
    \item \textbf{Framework:} MMSegmentation ~\cite{mmseg}.
    \item \textbf{Training Schedule:} 200 epochs.
    \item \textbf{Batch Size}: 16.
    \item \textbf{Crop Size:} $512 \times 512$.
    \item \textbf{Data Augmentation:} Same as the default settings of MMSegmentation. (Random Resize, Random Crop, Random Flip, and Photo Metric Distortion.)
\end{itemize}

\section{Additional Visualizations}

In this section, we provide extensive qualitative results to demonstrate the visual fidelity and semantic diversity of the synthetic datasets generated by S$^2$-RL. We showcase all generated images from both ADE20K and COCO datasets.

\subsection{Qualitative Results on ADE20K}
Figures~\ref{fig:ade_0.25}, \ref{fig:ade_0.5}, and \ref{fig:ade_1_2} display all synthetic images generated for the ADE20K dataset using the first generation seed. The results highlight our model's ability to synthesize complex scene layouts, such as cluttered bedrooms and diverse street scenes, strictly following the semantic constraints defined in the prompts.

\subsection{Qualitative Results on COCO}
Figures~\ref{fig:coco_0.625}, \ref{fig:coco_1.25}, and \ref{fig:coco_2.5_3} present the generation results on the COCO dataset using the first generation seed. Compared to ADE20K, COCO contains more object-centric scenes with varied backgrounds. S$^2$-RL effectively handles these variations, producing high-fidelity objects with diverse textures and lighting conditions.

\section{Extended Ablation Studies}

All experiments in this section use ADE20K at the 1.0\% budget. The coreset, reward-weight, group-size, prototype, and evaluator ablations use UPerNet with a ResNet-50 backbone; independent-policy evaluation uses SegFormer-B5. We distinguish variation across policy training runs from variation across downstream student seeds.

\subsection{DINOv3-controlled Coreset Reruns}
\label{app:dinov3_coreset}

We further study the influence of feature extractors in coreset selection. The standard K-Center and Herding baselines follow their conventional setup and use ResNet-50 features for selection. Since S$^2$-RL uses DINOv3 features inside the set-level reward computation, we additionally rerun K-Center and Herding with DINOv3 features while keeping the downstream training protocol unchanged.

This analysis is conducted on ADE20K at the 1.0\% budget under the canonical UPerNet + ResNet-50 downstream evaluation protocol. The purpose of this controlled experiment is to isolate the effect of the feature extractor used for subset selection. We independently train the same downstream segmentation architecture on each selected subset.

\begin{table}[ht]
\centering
\caption{DINOv3-controlled coreset reruns on ADE20K at the 1.0\% budget. This controlled analysis uses the canonical UPerNet + ResNet-50 downstream student.}
\label{tab:dinov3_coreset}
\begin{tabular}{lccc}
\toprule
Method & Selection Feature & mAcc (\%) & mIoU (\%) \\
\midrule
K-Center & ResNet-50 & 19.73$\pm$0.51 & 14.35$\pm$0.46 \\
K-Center & DINOv3 & 21.53$\pm$0.63 & 14.98$\pm$0.37 \\
Herding & ResNet-50 & 16.21$\pm$0.23 & 11.88$\pm$0.14 \\
Herding & DINOv3 & 17.55$\pm$0.26 & 12.66$\pm$0.15 \\
Random & N/A & 16.61$\pm$0.67 & 12.00$\pm$0.46 \\
Uniform & N/A & 21.03$\pm$0.92 & 14.83$\pm$0.57 \\
S$^2$-RL & DINOv3 reward only & \textbf{23.07$\pm$0.34} & \textbf{15.43$\pm$0.20} \\
\bottomrule
\end{tabular}
\end{table}

Replacing ResNet-50 with DINOv3 improves the two coreset baselines. K-Center improves from 14.35 to 14.98 mIoU, and Herding improves from 11.88 to 12.66 mIoU. This shows that stronger representation features can improve subset selection quality.

S$^2$-RL still achieves the best performance under the same downstream evaluation protocol. This indicates that the improvement of S$^2$-RL is not solely explained by the use of DINOv3 features. Instead, feature-based selection and reward-driven scene construction play different roles: DINOv3 improves the representation space for selecting existing samples, while S$^2$-RL uses the same type of semantic representation to guide generative exploration and construct new high-density training samples.

\subsection{Impact of Reward Balancing Weights}
The total reward in Eq.~\ref{eq:total} is a weighted sum of alignment ($\gamma_a$), diversity ($\gamma_d$), and scarcity ($\gamma_s$) terms. We investigate the sensitivity of the downstream performance to variations in these weights.

\textbf{Analysis.} As shown in Table~\ref{tab:sens_gamma}, S$^2$-RL demonstrates remarkable robustness to hyperparameter variations. While the optimal configuration ($\gamma_d=5.0, \gamma_s=3.0$) yields a peak mIoU of 15.43\%, deviating from these values results in only marginal fluctuations. Notably, even when the diversity weight is drastically increased to 10.0 or the scarcity weight is adjusted, the model consistently maintains a mIoU above 14.66\%, with a maximum performance drop of less than 0.8\%. This stability indicates that our framework relies on the intrinsic synergy of the proposed reward components rather than precise hyperparameter tuning, ensuring reliable performance across a broad range of settings.

\begin{table}[ht]
   \centering
   \caption{Impact of varying reward weights on downstream metrics (ADE20K 1\%).}
   \label{tab:sens_gamma}
   \begin{tabular}{ccc|cc}
   \toprule
   $\gamma_a$ & $\gamma_d$ & $\gamma_s$ & mAcc & mIoU \\
   \midrule
   1.0 & 3.0 & 3.0 & \res{21.65}{0.29} & \res{14.66}{0.16} \\
   1.0 & 3.0 & 5.0 & \res{22.40}{0.36} & \res{15.13}{0.16} \\
   1.0 & 5.0 & 3.0 & \textbf{\res{23.07}{0.34}} & \textbf{\res{15.43}{0.20}} \\
   1.0 & 5.0 & 5.0 & \res{22.47}{0.35} & \res{15.29}{0.19} \\
   1.0 & 10.0 & 3.0 & \res{22.18}{0.37} & \res{15.04}{0.25} \\
   1.0 & 10.0 & 5.0 & \res{22.91}{0.27} & \res{15.31}{0.23} \\
   \bottomrule
   \end{tabular}
\end{table}

\subsection{GRPO Group Size and Policy Variability}
\label{app:policy_stability}
The group size $G$ in GRPO determines the accuracy of the baseline estimation for variance reduction. We evaluate the model performance with $G \in \{4, 8, 16\}$.

\textbf{Analysis.} As shown in Table ~\ref{tab:sens_g}, increasing $G$ stabilizes the gradient estimate, leading to faster convergence and higher final mIoU. However, the performance saturates around $G=8$, after which the computational overhead outweighs the marginal gains. We choose $G=8$ for ADE20K; the COCO configuration uses $G=16$.

\begin{table}[ht]
    \centering
    \caption{Sensitivity to GRPO group size $G$ on downstream mIoU (ADE20K 1\%).}
    \label{tab:sens_g}
    \begin{tabular}{l|ccc}
    \toprule
    $G$ & 4 & 8 & 16 \\
    \midrule
    mIoU (\%) & \res{14.96}{0.25} & \res{15.43}{0.20} & \res{15.52}{0.19} \\
    \bottomrule
    \end{tabular}
\end{table}

\textbf{Independent policy runs.} We independently train three S$^2$-RL policies on ADE20K and evaluate their 1.0\% constructed sets with SegFormer-B5. After averaging downstream student-seed results within each policy, the mean and standard deviation across policies are $25.18\pm0.30$ mIoU and $35.90\pm0.77$ mAcc. This measures policy-training variation alongside the variation across constructed sets and student seeds reported in the main comparison.

\subsection{Prototype Number and Construction}
\label{app:prototype_k}
\label{app:prototype_construction}

We study the effect of the number of semantic prototypes, or semantic arms, used in the set-level reward. The main experiments use $K=50$ prototypes for each class. We additionally evaluate $K=25$ and $K=100$ on ADE20K at the 1.0\% budget.

\begin{table}[ht]
\centering
\caption{Effect of the number of semantic prototypes $K$ on ADE20K at the 1.0\% budget.}
\label{tab:prototype_k}
\begin{tabular}{ccc}
\toprule
$K$ & mAcc (\%) & mIoU (\%) \\
\midrule
25 & 22.30$\pm$0.38 & 15.00$\pm$0.16 \\
50 & \textbf{23.07$\pm$0.34} & \textbf{15.43$\pm$0.20} \\
100 & 22.58$\pm$0.87 & 15.24$\pm$0.55 \\
\bottomrule
\end{tabular}
\end{table}

The results show that S$^2$-RL is not highly sensitive to moderate changes in $K$. All three settings produce similar downstream performance, with a maximum difference of 0.43 mIoU. The default setting $K=50$ achieves the best result. Increasing $K$ from 50 to 100 does not further improve performance and leads to larger variance, suggesting that overly fine-grained prototype partitioning may make the set-level reward less stable under a small data budget.

In the current implementation, the same number of prototypes is used for each class. This simple design provides a uniform semantic-arm structure across categories, while the scarcity term in the set-level reward assigns larger bonuses to under-covered semantic modes during generation.

\textbf{Prototype construction.} At $K=50$, we compare nearest-observed K-means representatives with randomly selected same-class region features. Each representative is an observed region feature rather than a cluster-center vector. Both variants repeat policy adaptation from the same pretrained initialization, keeping the other reward and training settings fixed. Table~\ref{tab:prototype_construction} reports means and standard deviations over five downstream UPerNet--ResNet-50 seeds. The clustered representatives yield better downstream performance, suggesting that prototype placement complements the choice of bank size.

\begin{table}[ht]
\centering
\caption{Prototype construction at $K=50$ on ADE20K 1.0\%, evaluated with UPerNet--ResNet-50 over five student seeds.}
\label{tab:prototype_construction}
\begin{tabular}{lcc}
\toprule
Prototype bank & mIoU (\%) & mAcc (\%) \\
\midrule
Random same-class regions & \res{14.82}{0.37} & \res{22.21}{0.51} \\
K-means nearest real regions & \textbf{\res{15.43}{0.20}} & \textbf{\res{23.07}{0.34}} \\
\bottomrule
\end{tabular}
\end{table}

\subsection{Reward Evaluator Robustness}
\label{app:evaluator_robustness}

S$^2$-RL uses frozen segmentation models as reward evaluators during diffusion policy optimization. In the main setting, we use an ensemble evaluator to provide a more stable reward signal. We further study how the choice of reward evaluator affects the final distilled set by replacing the ensemble with a single frozen segmentation model.

This analysis keeps the generator prior, RL training protocol, reward formulation, and downstream evaluation protocol fixed. The only changed component is the reward evaluator used during policy optimization. All experiments are conducted on ADE20K at the 1.0\% budget using the canonical UPerNet + ResNet-50 downstream student.

\begin{table}[ht]
\centering
\caption{Reward-evaluator robustness under a fixed SD3.5 generator prior on ADE20K at the 1.0\% budget.}
\label{tab:evaluator_robustness}
\begin{tabular}{lc}
\toprule
Reward Evaluator(s) & mIoU (\%) \\
\midrule
UPerNet only & 14.52$\pm$0.22 \\
SegFormer-B5 only & 15.01$\pm$0.36 \\
DPT-Large only & 14.85$\pm$0.25 \\
Ensemble & \textbf{15.43$\pm$0.20} \\
\bottomrule
\end{tabular}
\end{table}

The results show that S$^2$-RL remains effective with different frozen evaluators under the same generator prior. Among the single-evaluator variants, SegFormer-B5 achieves the highest mIoU, followed by DPT-Large and UPerNet. The ensemble evaluator performs best, indicating that aggregating complementary segmentation priors provides a more reliable reward signal for policy optimization.

This controlled study separates evaluator choice from the generator prior. Under the fixed SD3.5 prior, the performance improvement of S$^2$-RL does not depend on a single reward-evaluator architecture. Instead, the reward-driven generation process remains effective across multiple segmentation evaluators, while the ensemble setting provides the strongest overall result.

\section{Independent Semantic and Layout Evaluation}
\label{app:semantic_audit}

\subsection{Protocol and Metric Definitions}

We evaluate five stored constructed sets per method at ADE20K 1.0\% (202 images per set) and COCO 2.5\textperthousand\ (295 images per set). A frozen OneFormer-Swin-L~\cite{oneformer} annotates both the constructed images and the official held-out validation reference (2,000 ADE20K and 5,000 COCO images). All four methods, including real-image Uniform selection, use this same annotation source. OneFormer was unused by the original reward, candidate-selection, pseudo-labeling, and student-training pipelines for these sets. This diagnostic is therefore independent of the reward evaluator, although it still relies on pretrained dataset priors.

For semantic appearance, we mask-pool ImageNet-1K V2 ResNet-50 layer-4 features extracted at $448\times448$ resolution; this feature extractor is also unused by the reward. We discard class regions smaller than 0.1\% of image area and repeat the analysis at 0.2\%. The OneFormer checkpoint revisions are fixed to \texttt{4a5bac8e64f82681a12db2e151a4c2f4ce6092b2} (ADE20K) and \texttt{3a263017ca5c75adbea145f25f81b118243d4394} (COCO).

\textbf{Local-support coverage and precision (C/P).} For a held-out descriptor, its self-excluded real 5-nearest-neighbor distance defines a fixed support radius. Coverage is the fraction of held-out descriptors with a constructed descriptor inside their support ball; precision is the fraction of constructed descriptors supported by the union of these fixed real-reference balls. Feature C/P applies this rule to mask-pooled ResNet features, and occupancy C/P to normalized $4\times4$ class-occupancy descriptors. Both are macro-averaged over reference-eligible classes. These are local-support diagnostics, distinct from standard generative precision/recall.

\textbf{Pair recall/precision and geometry C/P.} Pair recall and precision compare the unique category pairs co-occurring within images. For each category pair, geometry uses one descriptor per image from the two class-union masks: horizontal and vertical centroid displacement, log area ratio, and centroid distance normalized by the image diagonal. Geometry C/P applies the fixed-reference support rule separately to pairs with at least ten held-out observations and macro-averages over eligible pairs. These descriptors measure pairwise spatial support, not object boundaries, higher-order scene graphs, or full topology. We also report class-conditional Vendi~\cite{vendi} on the pooled ResNet features.

\textbf{Statistical scope.} Tables~\ref{tab:semantic_features} and~\ref{tab:semantic_geometry} give means and sample standard deviations across five stored sets. A content-hash check found no constructed/reference or cross-method image overlap, but images can be reused across sets of the same method. The reported variation is descriptive set-level variation, not independent-policy uncertainty or a significance test.

\subsection{Semantic and Pairwise Layout Results}

\begin{table}[ht]
\centering
\small
\setlength{\tabcolsep}{5pt}
\caption{Reward-independent feature and occupancy diagnostics at the 0.1\% area cutoff. Mean $\pm$ sample SD across five stored sets. C: coverage; P: precision.}
\label{tab:semantic_features}
\begin{tabular}{lccccc}
\toprule
Dataset/method & Feature C & Feature P & Occupancy C & Occupancy P & Vendi \\
\midrule
ADE Uniform & \res{.525}{.015} & \res{.878}{.010} & \res{.506}{.016} & \res{.929}{.010} & \res{2.664}{.023} \\
ADE SD3.5 & \res{.370}{.015} & \res{.777}{.018} & \res{.434}{.016} & \res{.888}{.019} & \res{2.397}{.005} \\
ADE UniDD$^\ast$ & \res{.443}{.018} & \res{.834}{.019} & \res{.427}{.011} & \res{.901}{.024} & \res{2.587}{.025} \\
ADE S$^2$-RL & \res{.447}{.010} & \res{.796}{.011} & \res{.538}{.016} & \res{.929}{.005} & \res{2.389}{.025} \\
COCO Uniform & \res{.314}{.011} & \res{.875}{.008} & \res{.304}{.014} & \res{.967}{.005} & \res{2.659}{.005} \\
COCO SD3.5 & \res{.240}{.005} & \res{.884}{.015} & \res{.262}{.003} & \res{.967}{.008} & \res{2.183}{.029} \\
COCO UniDD$^\ast$ & \res{.210}{.007} & \res{.818}{.008} & \res{.274}{.007} & \res{.971}{.006} & \res{2.527}{.031} \\
COCO S$^2$-RL & \res{.286}{.004} & \res{.869}{.009} & \res{.301}{.003} & \res{.964}{.007} & \res{2.215}{.017} \\
\bottomrule
\end{tabular}
\end{table}

\begin{table}[ht]
\centering
\small
\setlength{\tabcolsep}{9pt}
\caption{Reward-independent category-pair and relative-geometry diagnostics at the 0.1\% area cutoff. Mean $\pm$ sample SD across five stored sets.}
\label{tab:semantic_geometry}
\begin{tabular}{lcccc}
\toprule
Dataset/method & Pair R & Pair P & Geometry C & Geometry P \\
\midrule
ADE Uniform & \res{.592}{.008} & \res{.932}{.011} & \res{.544}{.012} & \res{.916}{.005} \\
ADE SD3.5 & \res{.539}{.019} & \res{.941}{.008} & \res{.549}{.020} & \res{.880}{.012} \\
ADE UniDD$^\ast$ & \res{.500}{.017} & \res{.955}{.002} & \res{.464}{.015} & \res{.862}{.016} \\
ADE S$^2$-RL & \res{.665}{.015} & \res{.898}{.010} & \res{.693}{.006} & \res{.939}{.004} \\
COCO Uniform & \res{.534}{.024} & \res{.975}{.005} & \res{.334}{.022} & \res{.836}{.023} \\
COCO SD3.5 & \res{.475}{.007} & \res{.981}{.002} & \res{.288}{.013} & \res{.778}{.010} \\
COCO UniDD$^\ast$ & \res{.476}{.010} & \res{.979}{.005} & \res{.277}{.011} & \res{.781}{.014} \\
COCO S$^2$-RL & \res{.578}{.018} & \res{.970}{.009} & \res{.424}{.009} & \res{.872}{.015} \\
\bottomrule
\end{tabular}
\end{table}

S$^2$-RL has the highest category-pair recall and pair-macro geometry coverage and precision on both datasets. This broader pair coverage comes with lower pair precision. Uniform has the strongest generic feature coverage and slightly higher COCO occupancy coverage ($.304$ versus $.301$). Vendi is mixed: S$^2$-RL is slightly below SD3.5 on ADE20K and slightly above it on COCO, while remaining below Uniform and UniDD$^\ast$ on both. These results support broader measured category-pair and relative-geometry support rather than universal diversity improvement.

\textbf{Threshold and aggregation sensitivity.} Table~\ref{tab:semantic_sensitivity} combines the stricter 0.2\% area cutoff with alternative precision aggregations. The ordering in pair recall and pair-macro geometry C/P persists at the stricter cutoff. Micro precision counts every constructed-set observation, including unsupported classes or pairs; real-frequency-weighted geometry precision weights eligible pairs by held-out occurrence frequency. Geometry precision favors S$^2$-RL under pair-macro and real-frequency weighting, while micro-precision results are mixed.

\begin{table}[ht]
\centering
\small
\setlength{\tabcolsep}{4pt}
\caption{Threshold and aggregation sensitivity. Pair R/P and geometry C/P use the 0.2\% area cutoff; the three precision columns list 0.1\% / 0.2\%. Values are descriptive five-set means.}
\label{tab:semantic_sensitivity}
\begin{tabular}{lccccc}
\toprule
Dataset/method & \makecell{Pair\\R/P} & \makecell{Geometry\\C/P} & \makecell{Feature\\micro P} & \makecell{Geometry\\micro P} & \makecell{Real-weighted\\geometry P} \\
\midrule
ADE Uniform & .571/.926 & .537/.907 & .851/.852 & .822/.804 & .959/.955 \\
ADE SD3.5 & .529/.935 & .545/.879 & .792/.793 & .856/.833 & .948/.948 \\
ADE UniDD$^\ast$ & .492/.951 & .475/.867 & .839/.837 & .857/.838 & .940/.941 \\
ADE S$^2$-RL & .645/.896 & .680/.937 & .795/.800 & .822/.799 & .971/.971 \\
COCO Uniform & .517/.972 & .336/.836 & .868/.866 & .918/.909 & .946/.944 \\
COCO SD3.5 & .464/.977 & .296/.778 & .865/.864 & .927/.914 & .924/.919 \\
COCO UniDD$^\ast$ & .460/.977 & .283/.785 & .814/.811 & .929/.922 & .929/.927 \\
COCO S$^2$-RL & .568/.965 & .429/.873 & .825/.823 & .919/.906 & .958/.957 \\
\bottomrule
\end{tabular}
\end{table}

\textbf{Complementary image-level diagnostics.} Table~\ref{tab:generic_diagnostics} reports FID~\cite{fid}, KID~\cite{kid}, and PRDC~\cite{prdc} from InceptionV3 pool-3 features. FID uses full-training-set reference statistics; KID and PRDC use the held-out reference, with 50 KID subsets of size 100 and PRDC $k=5$. Matched LPIPS~\cite{lpips} uses AlexNet (v0.1) on up to 100 disjoint image pairs with mask class-set Jaccard similarity at least 0.3, using the original teacher masks for synthetic images and ground-truth masks for Uniform. Results are means and sample standard deviations across the same five stored sets. These metrics do not consistently track segmentation utility: relative to unadapted SD3.5, S$^2$-RL has worse ADE20K FID/KID, modestly better COCO distribution scores, and lower matched LPIPS on both datasets. The semantic/layout gains therefore coexist with mixed appearance-space results.

\begin{table}[ht]
\centering
\small
\setlength{\tabcolsep}{4pt}
\caption{Generic image-distribution and matched perceptual-diversity diagnostics. Lower FID/KID and higher recall, coverage, and LPIPS are preferred by these metrics.}
\label{tab:generic_diagnostics}
\begin{tabular}{lccccc}
\toprule
Dataset/method & FID & KID $\times10^3$ & PRDC recall & PRDC coverage & Matched LPIPS \\
\midrule
ADE Uniform & \res{101.31}{0.74} & \res{1.83}{0.66} & \res{0.791}{0.014} & \res{0.381}{0.019} & \res{0.6996}{0.0043} \\
ADE SD3.5 & \res{118.08}{1.72} & \res{10.10}{1.24} & \res{0.592}{0.012} & \res{0.265}{0.009} & \res{0.6566}{0.0048} \\
ADE UniDD$^\ast$ & \res{106.75}{1.68} & \res{4.49}{1.25} & \res{0.749}{0.016} & \res{0.352}{0.014} & \res{0.6993}{0.0052} \\
ADE S$^2$-RL & \res{122.21}{1.05} & \res{15.71}{1.27} & \res{0.550}{0.042} & \res{0.264}{0.027} & \res{0.5754}{0.0023} \\
COCO Uniform & \res{108.20}{1.58} & \res{0.99}{0.36} & \res{0.791}{0.012} & \res{0.253}{0.020} & \res{0.7043}{0.0113} \\
COCO SD3.5 & \res{120.98}{1.33} & \res{9.00}{0.89} & \res{0.587}{0.023} & \res{0.218}{0.012} & \res{0.6867}{0.0087} \\
COCO UniDD$^\ast$ & \res{110.53}{0.87} & \res{3.03}{0.27} & \res{0.747}{0.011} & \res{0.252}{0.012} & \res{0.7090}{0.0041} \\
COCO S$^2$-RL & \res{120.89}{1.48} & \res{8.13}{0.98} & \res{0.606}{0.016} & \res{0.222}{0.011} & \res{0.6292}{0.0033} \\
\bottomrule
\end{tabular}
\end{table}

\section{Limitations and Future Work}

While S$^2$-RL establishes a new state-of-the-art for semantic segmentation dataset distillation, we identify several limitations inherent to the current framework and outline promising directions for future research.

\subsection{Limitations}

\textbf{Dataset Construction Cost.} S$^2$-RL incurs an upfront policy-adaptation cost of approximately 48 GPU-hours, followed by multi-step diffusion sampling. Selection-based methods are substantially cheaper to construct. Although policy optimization avoids differentiating through dense student-training updates, the intended use is offline extreme compression with repeated reuse of a constructed set; the measured generative budget does not establish an end-to-end lifecycle speedup.

\textbf{Dependency on pretrained Priors.} Our framework relies on the rich semantic priors encapsulated in the pretrained Stable Diffusion 3.5 and the frozen segmentor reward model. Consequently, the performance of S$^2$-RL is bounded by the capabilities of these foundation models. For extremely rare categories not well-represented in the Stable Diffusion training corpus, the policy may struggle to generate high-fidelity textures even with RL fine-tuning.

\textbf{Pseudo-Label Fidelity.} Generated training masks can inherit teacher errors, especially for small or ambiguous regions.

\textbf{Reward Proxy Gap.} We utilize a frozen segmentor as a proxy to estimate the downstream utility of generated samples. While effective, this is an approximation of the true student learning dynamics. There exists a theoretical gap between maximizing the frozen segmentor's confidence and minimizing the student's validation loss. In rare cases, the agent might learn to hack the reward by generating adversarial features that satisfy the frozen segmentor but are less transferable to a student model trained independently on the constructed dataset.

\subsection{Future Work}

\textbf{Extension to Other Dense Prediction Tasks.} The proposed sample-set dual reward paradigm is generic and not limited to semantic segmentation. Future work could extend S$^2$-RL to other dense prediction tasks, such as instance segmentation by incorporating box-level alignment rewards and depth estimation by replacing the semantic evaluator with a depth consistency reward.

\textbf{Student-in-the-Loop Optimization.} To bridge the reward proxy gap mentioned above, future iterations could incorporate feedback directly from the student model being trained. Although calculating gradients through the student involves high memory costs (the original bi-level bottleneck), one could employ meta-learning techniques or periodic reward model updates to approximate the student's needs without full gradient unrolling.

\section{Broader Impacts}
\label{app:broader_impacts}

This work studies semantic segmentation dataset distillation under compact data budgets. Its primary positive impact is improving the efficiency of dense prediction research by reducing the number of annotated training samples required for downstream model development. By constructing compact synthetic training sets, S$^2$-RL may lower storage, annotation, and iteration costs for semantic segmentation experiments, which can make dense prediction model development more accessible in resource-constrained settings.

The method also inherits several limitations and risks from pretrained generative and segmentation models. Since S$^2$-RL relies on pretrained text-to-image generators and frozen reward evaluators, the generated samples may reflect biases, artifacts, or category imbalances present in these models or in the original training datasets. If such distilled sets are used without careful validation, downstream models may inherit these biases or perform unevenly across rare or visually ambiguous categories. Therefore, the generated data should be treated as a compact training surrogate for research evaluation rather than as a replacement for carefully curated real annotations in safety-critical applications.

\begin{figure}[p]
    \centering
    \includegraphics[width=1.0\linewidth]{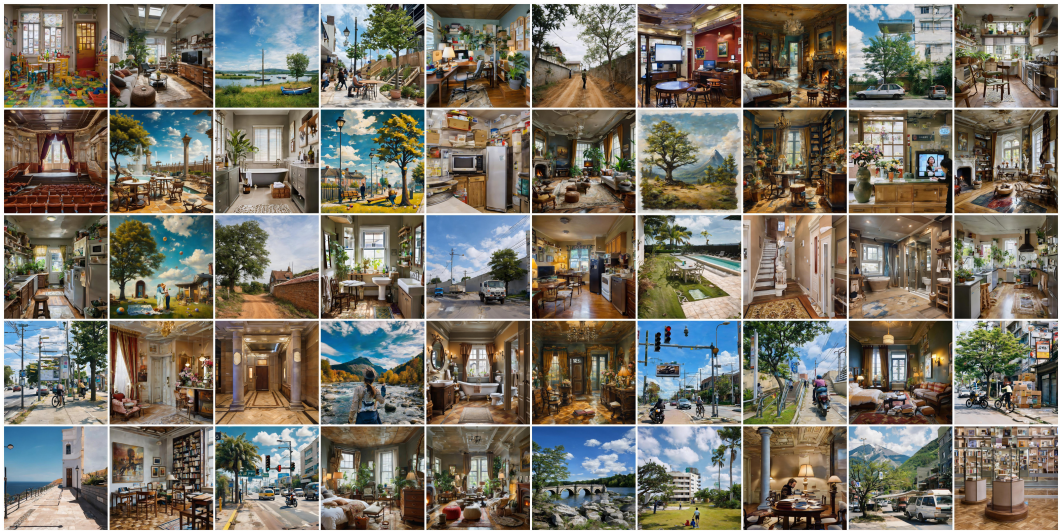}
    \caption{Generated samples on ADE20K with 0.25\% data ratio.}
    \label{fig:ade_0.25}
\end{figure}

\begin{figure}[p]
    \centering
    \includegraphics[width=1.0\linewidth]{fig/time_1_ratio_0.005.pdf}
    \caption{Generated samples on ADE20K with 0.5\% data ratio.}
    \label{fig:ade_0.5}
\end{figure}

\begin{figure}[p]
    \centering
    \includegraphics[width=1.0\linewidth]{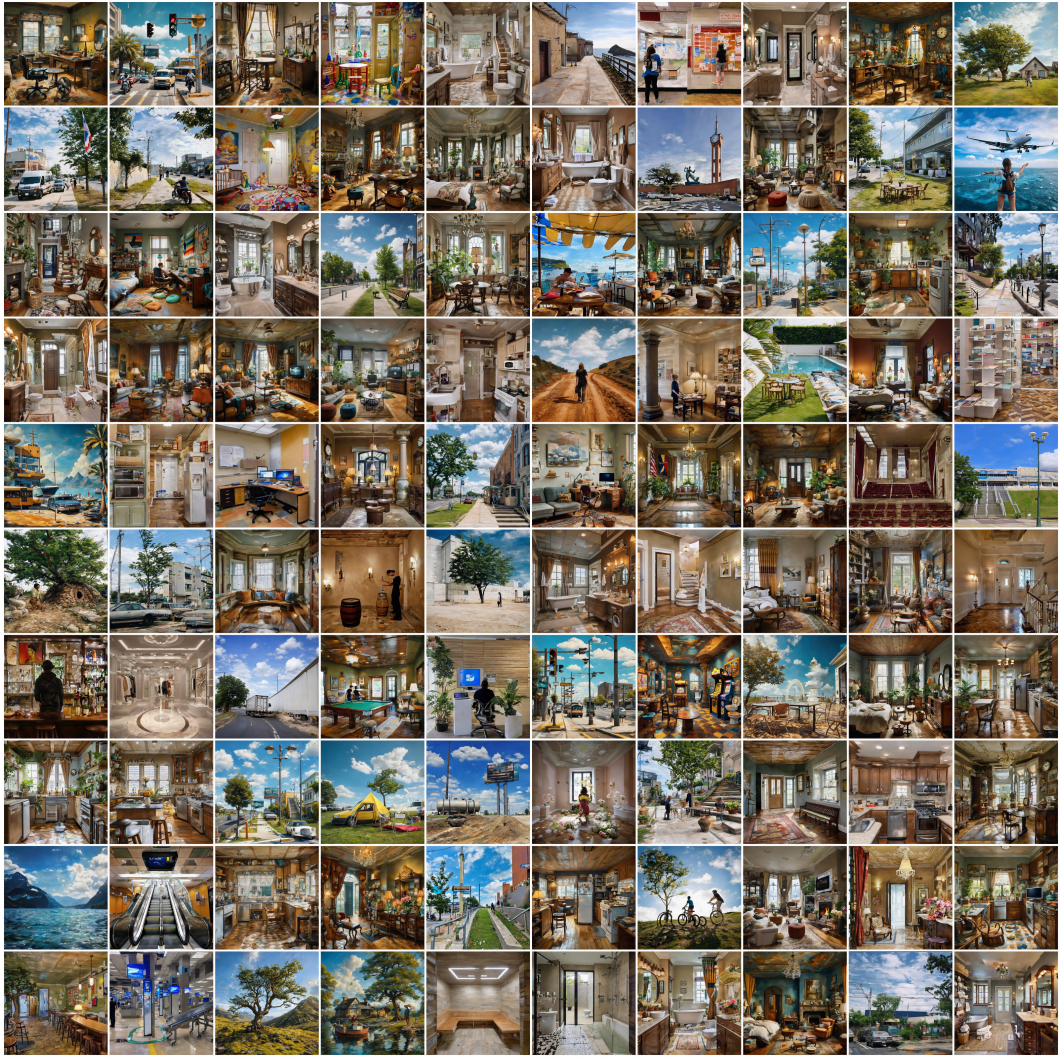}
    \caption{Generated samples on ADE20K with 1.0\% data ratio. (Continued on next page)}
    \label{fig:ade_1_1}
\end{figure}

\begin{figure}[p]
    \ContinuedFloat 
    \centering
    \includegraphics[width=1.0\linewidth]{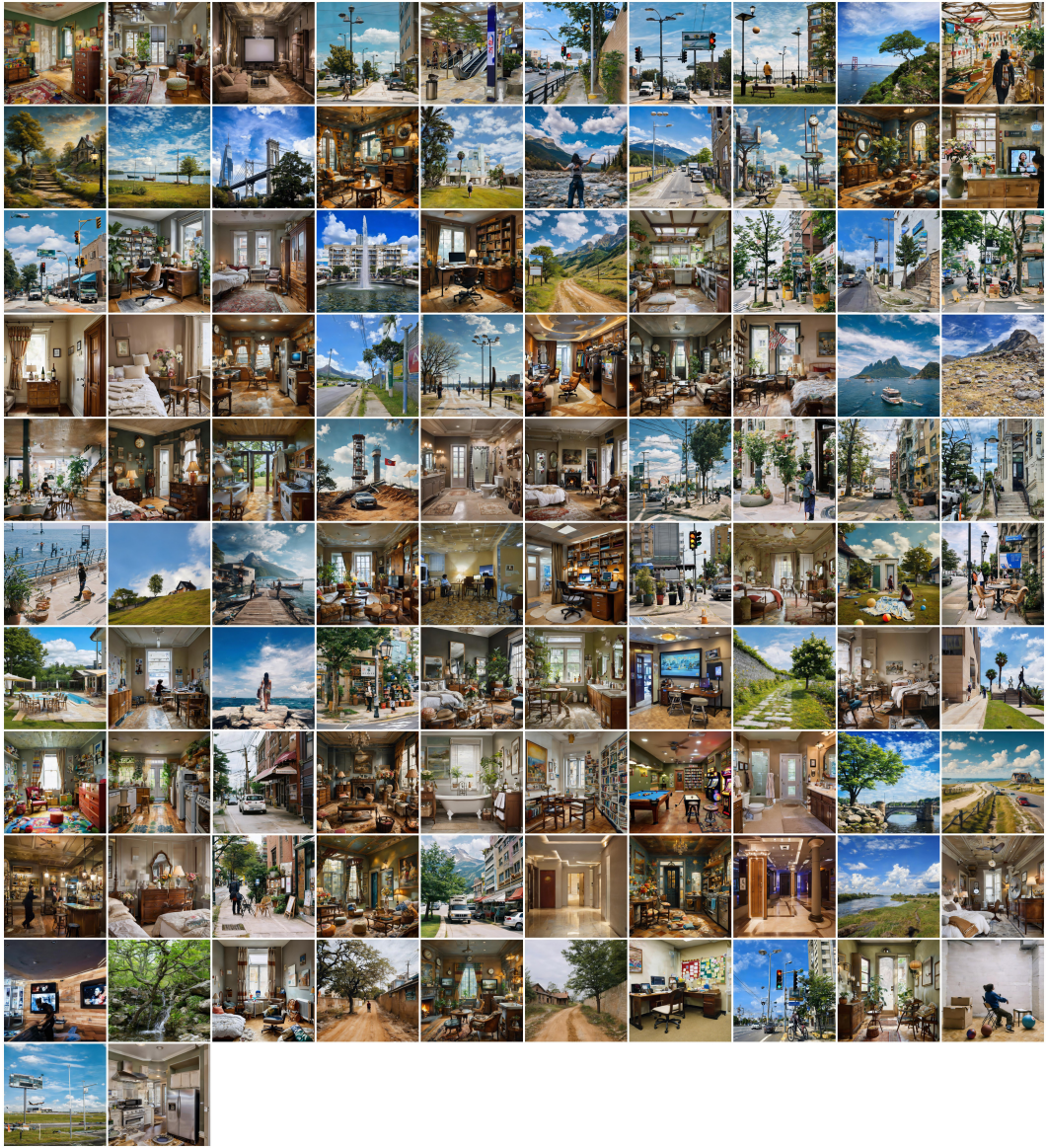}
    \caption{(Continued) Generated samples on ADE20K with 1.0\% data ratio.}
    \label{fig:ade_1_2}
\end{figure}

\begin{figure}[p]
    \centering
    \includegraphics[width=1.0\linewidth]{fig/coco_time_1_ratio_0.000625.pdf}
    \caption{Generated samples on COCO with 0.625\textperthousand\ data ratio.}
    \label{fig:coco_0.625}
\end{figure}

\begin{figure}[p]
    \centering
    \includegraphics[width=1.0\linewidth]{fig/coco_time_1_ratio_0.00125.pdf}
    \caption{Generated samples on COCO with 1.25\textperthousand\ data ratio.}
    \label{fig:coco_1.25}
\end{figure}

\begin{figure}[p]
    \centering
    \includegraphics[width=1.0\linewidth]{fig/coco_time_1_ratio_0.0025_part1.pdf}
    \caption{Generated samples on COCO with 2.5\textperthousand\ data ratio. (Continued on next page)}
    \label{fig:coco_2.5_1}
\end{figure}

\begin{figure}[p]
    \ContinuedFloat 
    \centering
    \includegraphics[width=1.0\linewidth]{fig/coco_time_1_ratio_0.0025_part2.pdf}
    \caption{(Continued) Generated samples on COCO with 2.5\textperthousand\ data ratio. (Continued on next page)}
    \label{fig:coco_2.5_2}
\end{figure}

\begin{figure}[p]
    \ContinuedFloat 
    \centering
    \includegraphics[width=1.0\linewidth]{fig/coco_time_1_ratio_0.0025_part3.pdf}
    \caption{(Continued) Generated samples on COCO with 2.5\textperthousand\ data ratio.}
    \label{fig:coco_2.5_3}
\end{figure}

\clearpage
\newpage
\section*{NeurIPS Paper Checklist}

\begin{enumerate}

\item {\bf Claims}
    \item[] Question: Do the main claims made in the abstract and introduction accurately reflect the paper's contributions and scope?
    \item[] Answer: \answerYes{} %
    \item[] Justification: The abstract and introduction state the main contributions of S$^2$-RL, including the reformulation of semantic segmentation dataset distillation as reward-driven diffusion policy optimization and the proposed sample-set dual reward design. The claims are supported by experiments on ADE20K and COCO and are scoped to semantic segmentation dataset distillation under compact data budgets.

\item {\bf Limitations}
    \item[] Question: Does the paper discuss the limitations of the work performed by the authors?
    \item[] Answer: \answerYes{} %
    \item[] Justification: The main text discusses upfront construction cost and possible pseudo-label errors. The appendix further discusses dependence on pretrained models and the reward proxy gap between frozen evaluator scores and downstream student learning dynamics.

\item {\bf Theory assumptions and proofs}
    \item[] Question: For each theoretical result, does the paper provide the full set of assumptions and a complete (and correct) proof?
    \item[] Answer: \answerYes{} %
    \item[] Justification: The paper presents theoretical propositions as optimization interpretations. The main text states the high-level proposition, and the appendix provides detailed derivations and supporting propositions for GRPO gradient estimation and the sample-set dual reward mechanism.

    \item {\bf Experimental result reproducibility}
    \item[] Question: Does the paper fully disclose all the information needed to reproduce the main experimental results of the paper to the extent that it affects the main claims and/or conclusions of the paper (regardless of whether the code and data are provided or not)?
    \item[] Answer: \answerYes{} %
    \item[] Justification: The paper specifies datasets, compression ratios, baselines, downstream architectures, evaluation metrics, training schedules, reward definitions, and key hyperparameters. Additional implementation details, UniDD$^\ast$ baseline settings, computational resources, and ablation protocols are provided in the appendix.

\item {\bf Open access to data and code}
    \item[] Question: Does the paper provide open access to the data and code, with sufficient instructions to faithfully reproduce the main experimental results, as described in supplemental material?
    \item[] Answer: \answerYes{} %
    \item[] Justification: Code and reproduction instructions are available at \url{https://github.com/00why00/S2RL}. Released datasets and policy weights are linked from the repository and hosted at \url{https://huggingface.co/0why0/S2RL}. The README covers data preparation, training-set construction, and downstream evaluation.

\item {\bf Experimental setting/details}
    \item[] Question: Does the paper specify all the training and test details (e.g., data splits, hyperparameters, how they were chosen, type of optimizer) necessary to understand the results?
    \item[] Answer: \answerYes{} %
    \item[] Justification: The experimental section describes the datasets, compression ratios, metrics, baselines, downstream training protocol, optimizer, crop size, and training schedule. The appendix further reports RL hyperparameters, reward computation details, downstream training details, and computational resource usage.

\item {\bf Experiment statistical significance}
    \item[] Question: Does the paper report error bars suitably and correctly defined or other appropriate information about the statistical significance of the experiments?
    \item[] Answer: \answerYes{} %
    \item[] Justification: The main comparisons report mean $\pm$ standard deviation over five constructed sets and five student seeds per set. The appendix separately reports three independent policy runs, five-seed prototype ablations, and descriptive variation across five stored sets for semantic/layout diagnostics. The stored-set diagnostics are not treated as independent policy replications or significance tests.

\item {\bf Experiments compute resources}
    \item[] Question: For each experiment, does the paper provide sufficient information on the computer resources (type of compute workers, memory, time of execution) needed to reproduce the experiments?
    \item[] Answer: \answerYes{} %
    \item[] Justification: The appendix includes a computational resources table reporting construction cost and peak GPU memory for selection-based baselines, SD3.5, UniDD$^\ast$, and S$^2$-RL. It also states the GPU type used for S$^2$-RL fine-tuning and reports the training and inference costs needed to construct the distilled sets.
    
\item {\bf Code of ethics}
    \item[] Question: Does the research conducted in the paper conform, in every respect, with the NeurIPS Code of Ethics \url{https://neurips.cc/public/EthicsGuidelines}?
    \item[] Answer: \answerYes{} %
    \item[] Justification: The work uses public datasets and pretrained models for semantic segmentation dataset distillation and does not involve human-subject experiments, private data collection, or harmful data acquisition. The study is conducted as a technical investigation of compact training set construction.

\item {\bf Broader impacts}
    \item[] Question: Does the paper discuss both potential positive societal impacts and negative societal impacts of the work performed?
    \item[] Answer: \answerYes{} %
    \item[] Justification: The paper discusses potential positive and negative societal impacts in Appendix. The positive impacts include reducing annotation, storage, and iteration costs for dense prediction research, while the negative impacts include possible inherited biases, artifacts, and pseudo-label errors from pretrained generative and segmentation models.
    
\item {\bf Safeguards}
    \item[] Question: Does the paper describe safeguards that have been put in place for responsible release of data or models that have a high risk for misuse (e.g., pre-trained language models, image generators, or scraped datasets)?
    \item[] Answer: \answerNA{} %
    \item[] Justification: The paper does not release a new general-purpose image generator, scraped dataset, or high-risk pretrained model. The experiments rely on existing public datasets and pretrained models, and the generated samples are used for semantic segmentation dataset distillation evaluation.

\item {\bf Licenses for existing assets}
    \item[] Question: Are the creators or original owners of assets (e.g., code, data, models), used in the paper, properly credited and are the license and terms of use explicitly mentioned and properly respected?
    \item[] Answer: \answerYes{} %
    \item[] Justification: The paper cites the datasets, pretrained models, and software frameworks used in the experiments, including ADE20K, COCO, Stable Diffusion 3.5, FreestyleNet, DINOv3, segmentation models, and MMSegmentation. We follow their intended research-use settings and provide code that relies on users obtaining the corresponding assets from their official sources and respecting their licenses.

\item {\bf New assets}
    \item[] Question: Are new assets introduced in the paper well documented and is the documentation provided alongside the assets?
    \item[] Answer: \answerYes{} %
    \item[] Justification: The implementation is released at \url{https://github.com/00why00/S2RL}, with a README covering dependencies, data preparation, and commands for training-set construction and downstream evaluation.

\item {\bf Crowdsourcing and research with human subjects}
    \item[] Question: For crowdsourcing experiments and research with human subjects, does the paper include the full text of instructions given to participants and screenshots, if applicable, as well as details about compensation (if any)? 
    \item[] Answer: \answerNA{} %
    \item[] Justification: The paper does not involve crowdsourcing, user studies, or research with human subjects.

\item {\bf Institutional review board (IRB) approvals or equivalent for research with human subjects}
    \item[] Question: Does the paper describe potential risks incurred by study participants, whether such risks were disclosed to the subjects, and whether Institutional Review Board (IRB) approvals (or an equivalent approval/review based on the requirements of your country or institution) were obtained?
    \item[] Answer: \answerNA{} %
    \item[] Justification: The paper does not involve human subjects, crowdsourcing, or collection of personal data, so IRB approval or equivalent human-subjects review is not applicable.

\item {\bf Declaration of LLM usage}
    \item[] Question: Does the paper describe the usage of LLMs if it is an important, original, or non-standard component of the core methods in this research? Note that if the LLM is used only for writing, editing, or formatting purposes and does \emph{not} impact the core methodology, scientific rigor, or originality of the research, declaration is not required.
    \item[] Answer: \answerNA{} %
    \item[] Justification: The core method does not use LLMs as an important, original, or non-standard component. Text prompts are constructed from semantic category names rather than generated or optimized by an LLM.

\end{enumerate}

\end{document}